\documentclass[sigconf]{acmart}

\usepackage{url}
\usepackage[utf8]{inputenc} 
\usepackage[T1]{fontenc}

\usepackage{hyperref}       
\usepackage{url}            
\usepackage{booktabs}       
\usepackage{amsfonts}       
\usepackage{nicefrac}       
\usepackage{microtype}      
\usepackage{xcolor}         
\usepackage{amsmath}
\usepackage{mathtools}
\usepackage{amsthm}
\usepackage{multirow}
\usepackage{makecell}
\usepackage[ruled,vlined,linesnumbered]{algorithm2e}
\usepackage{microtype}
\usepackage{graphicx}
\usepackage{booktabs} 
\usepackage{array}
\usepackage{textcomp}
\usepackage{stfloats}
\usepackage{verbatim}
\usepackage{balance}
\usepackage{caption}
\usepackage{tabularx}

\usepackage{graphicx}
\usepackage{subfig}

\AtBeginDocument{%
  }

\setcopyright{acmlicensed}
\copyrightyear{2018}
\acmYear{2018}
\acmDOI{XXXXXXX.XXXXXXX}

\acmConference[GraphSubDetector]{ }{Nov}{2024 }

\acmISBN{978-1-4503-XXXX-X/18/06}

\begin{document}
\title{GraphSubDetector: Time Series Subsequence Anomaly Detection via Density-Aware Adaptive Graph Neural Network}



\author{Weiqi~Chen, Zhiqiang~Zhou, 
        Qingsong~Wen, and 
        Liang~Sun}
\affiliation{
  \institution{DAMO Academy, Alibaba Group}\city{Hangzhou}\country{China}
}
\email{{jarvus.cwq,zhouzhiqiang.zzq,qingsong.wen,liang.sun}@alibaba-inc.com}

\renewcommand{\shortauthors}{Chen et al.}

\begin{abstract}
Time series subsequence anomaly detection is an important task in a large variety of real-world applications ranging from health monitoring to AIOps, and is challenging due to the following reasons: 1) how to effectively learn complex dynamics and dependencies in time series; 2) diverse and complicated anomalous subsequences as well as the inherent variance and noise of normal patterns; 3) how to determine the proper subsequence length for effective detection, which is a required parameter for many existing algorithms. 
In this paper, we present a novel approach to subsequence anomaly detection, namely GraphSubDetector. First, it adaptively learns the appropriate subsequence length with a length selection mechanism that highlights the characteristics of both normal and anomalous patterns. Second, we propose a density-aware adaptive graph neural network (DAGNN), which can generate further robust representations against variance of normal data for anomaly detection by message passing between subsequences.
The experimental results demonstrate the effectiveness of the proposed algorithm, which achieves superior performance on multiple time series anomaly benchmark datasets compared to state-of-the-art algorithms. 
\end{abstract}




\maketitle

\section{Introduction} \label{sec:intro}

{D}{etecting} anomalies in time series data has a wide variety of practical applications~\cite{wen2022robust}, such as tracing patients' biosignals for disease detection~\cite{chauhan2015anomaly}, monitoring operational data of cloud infrastructure for malfunction location~\cite{zhang2021cloudrca}, finding risks in IoT sensing time series~\cite{cook2019anomaly}, etc. It has received a great amount of research interest in the literature~\cite{li2020anomaly,hot:sax:2005, discord:2007, boniol13series2graph, shen2020timeseries,gao2020robusttad, matrix:profile:kdd22}. 
Most time series anomaly detection (TSAD) algorithms try to locate anomalies at each point of the time series (namely point-wise TSAD). However, this formulation fails to consider temporal relationships of anomalous points, as anomalies can go beyond occurring point by point but tend to exist consecutively over a time interval in many real-world scenarios. For example, some demand patterns of the electric power system change during the holidays. Figure~\ref{fig:point vs subseq} shows a comparison of point-wise anomalies and subsequence anomalies. In this paper, we investigate TSAD from a subsequence perspective by identifying anomalous patterns in a time interval, which is called time series subsequence anomaly detection. Generally speaking, a subsequence anomaly is a sequence of observations that deviates considerably from some concept of normality. The somewhat ``vague'' definition itself also hints the challenges of the subsequence anomaly detection problem.

\begin{figure*}
  \begin{minipage}[h!]{0.49\linewidth}
    \centering
    \includegraphics[width=1.\textwidth]{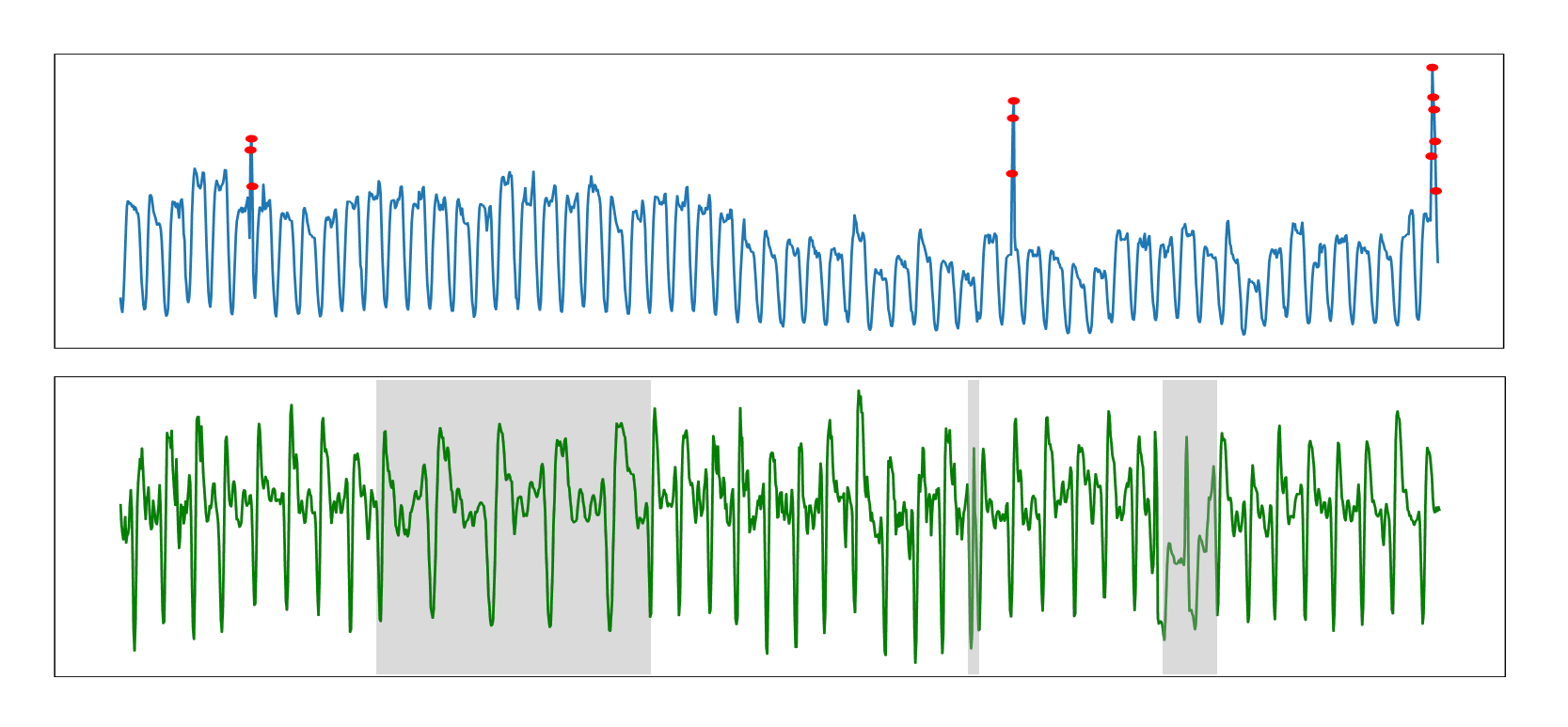}\vspace{-1mm}
    \caption{\small{Point-wise anomalies (top) versus subsequence anomalies (bottom). The top is a website traffic time series with anomalies labeled by red dots that might be caused by cyberattacks. The bottom is an insect's activity signal recorded with an EPG apparatus, where time intervals marked in grey are subsequences exhibiting different anomalous characteristics, including period length variation, spike, and temporal morphological change. Patterns and duration of anomalies vary. }}
    \label{fig:point vs subseq}
  \end{minipage}
  \hspace{2mm}
  \begin{minipage}[h!]{0.49\linewidth}
    \centering
    \includegraphics[width=1.\textwidth]{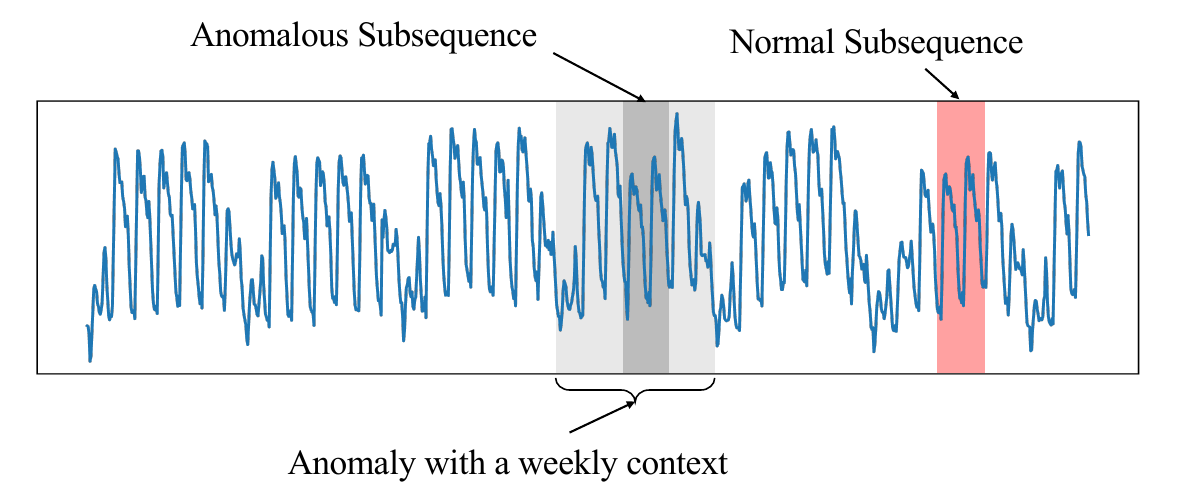}\vspace{-1mm}
    \caption{\small{An illustration of the difficulty in selecting proper subsequence length for subsequence anomaly detection. This figure shows an electricity consumption time series with both daily and weekly periods, and a 2-day anomalous subsequence that might be caused by power rationing inside the dark grey zone. If we directly detect anomalies using this length, the anomaly might not be found as it is very similar to normal subsequences, e.g., the green zone. Instead, it is better to select a longer length of a week (light grey), including the anomaly with its context to highlight the anomalous trend change.}}
    \label{fig:subseq_anomaly}
  \end{minipage}
\end{figure*}



Early research on anomaly detection relies mainly on shallow representations~\cite{breunig2000lof, scholkopf2001estimating, tax2004support}. Later, Deep-SVDD~\cite{ruff2018deep} enhances the representation learning capability using neural networks. Recently, TSAD methods ~\cite{ncad_ijcai22,shen2020timeseries} based on Deep-SVDD are prevailing due to their excellent performance. They introduce a suitable neural architecture for modeling time series and detecting anomalies by computing the distance between a target and its corresponding reference in the latent representation space, where the reference represents normal patterns. The main challenge is that these deep models are difficult to enforce assumptions of anomalies, and typically require large training datasets to achieve accurate results.
In contrast, time series discord~\cite{hot:sax:2005, yeh2016matrix, nakamura2020merlin, matrix:profile:kdd22} is another popular category of distance-based TSAD methods. Discords are subsequences that are maximally different from all others in the time series, where the difference is measured via zero-normalized Euclidean distance.
The most appealing merit of discords is that anomalies can be discovered by merely examining the test data without a training phase. Despite (or perhaps because of their extremely simple assumptions, discords are competitive with deep learning methods. However, several important limitations still prevent them from broader applications. First, they fail to detect similar anomalous patterns recurring at least twice, as each occurrence will be the others' nearest neighbor. Second, they rely on an explicit distance measure (z-normalized Euclidean distance), which cannot account for diversified anomalies flexibly and adaptively.

 

We summarize three major challenges in developing an effective and robust subsequence anomaly detector as follows. 

\noindent\textbf{Capturing the temporal dependency from time series data.} \hspace{1mm} A distinguishing feature of time series is the complex underlying dynamics and temporal dependencies. A temporal anomaly highly depends on the ``context", which should incorporate temporal dynamics and dependencies into consideration. 
Thus, how to learn and utilize temporal dependency is a key challenge in time series subsequence anomaly detection.

\noindent\textbf{Appropriate subsequence length selection.} \hspace{1mm} Another key challenge is how to determine the appropriate subsequence length in subsequence anomaly detection. Here we emphasize the importance of the appropriate length which highlights the normal or anomalous pattern of a subsequence. On the one hand, the duration of anomalies varies. For example, if we use a large window to detect spikes and dips in Figure~\ref{fig:point vs subseq}, the anomalies might be ``hidden'' in the normal data. While for a long-term anomaly, a short window cannot depict the full picture of it. On the other hand, even if we have a prior anomaly length, it is still necessary to intelligently infer a suitable length according to data characteristics, as illustrated in Figure \ref{fig:subseq_anomaly}. This problem becomes even worse when there are multiple abnormal subsequences with different lengths in one series. To the best of our knowledge, most existing methods utilize a predefined window length to detect anomalies and thus cannot intelligently detect anomalous subsequences with different lengths and characteristics. 


\noindent\textbf{Robust representations against inherent normal variance.} \hspace{1mm} Different time series data arising from different applications may exhibit different variances and noises.  
Thus, normal patterns can be flexible and volatile which hinders effectively and robustly distinguishing anomalies from normal data. Even though data transformation or representation learning in different detection algorithms have been proposed and some are resistant to noise to some extent, an inductive bias of variance removal of normal data is still necessary to be incorporated. 

In this paper, we devote to solving the aforementioned challenges in TSAD with the proposed GraphSubDetector, a graph neural network-based subsequence anomaly detection approach. Here we highlight our main contributions as follows:

1) In order to intelligently learn subsequence representations with proper length, we introduce a temporal convolution network (TCN) based feature encoder to generate representations from a multi-length view, and a length selection mechanism is utilized to determine the proper length for better depicting both normal and anomalous characteristics.

2) To further generate robust and effective representations for anomaly detection, we encode subsequence proximities into a graph and theoretically analyze that, with properly designed adjacency matrix, message passing between similar subsequence representations can reduce the variance of normal data and retain discrepancy between normal and anomalous data which makes anomalies prominent. Based on our analysis, we propose a novel density-aware adaptive graph neural network (DAGNN). Specifically, it adaptively learns the message passing adjacency matrix by incorporating temporal information and distance between subsequences in both data and latent space, and then refines it using local density information. By performing graph neural network on the learned graph, we can learn representations which distinguish normal and anomalous subsequences better. 

3) We have conducted experiments on multiple real-world TSAD benchmark datasets, and GraphSubDetector consistently outperforms state-of-the-art baselines. In particular, the proposed GraphSubDetector can effectively and efficiently detect different anomalous subsequences, which greatly improves its usage in practice. In terms of efficiency, the computational complexity of GraphSubDetector increases almost linearly as the data size increases.






\vspace{-2mm} 
\section{Related Work} \label{appendix:related works}
Subsequence anomaly detection in time series remains a challenging problem in the time series community. In this section, we briefly introduce some state-of-the-art algorithms for subsequence time series anomaly detection, including discord based methods, one-class based methods, reconstruction based methods, and other methods. 


\subsection{Anomaly Detection based on Discords} Subsequence anomaly detection can be formulated by finding discords in time series, where the discords are defined as the subsequences maximally different from other subsequences in the same time series~\cite{hot:sax:2005}. In \cite{discord:2007}, it is defined as the subsequence with the highest distance to its nearest neighbors. The distances can be calculated directly on the original signal~\cite{discord:2007}, or the representations like wavelets~\cite{discord:wavelet:2006}. Many methods in this category have been proposed, like Matrix Profile~\cite{yeh2016matrix} and its latest variants~\cite{nakamura2020merlin,matrix:profile:kdd22}.  Matrix Profile turns out to be an efficient solution for easy settings due to its efficient distance computation implementation. One limitation of the Matrix Profile is the challenging parameter tuning. For example, its performance degrades significantly when the window size is not properly set. Although it can easily capture the most different anomaly, it fails to detect similar anomalies which recur multiple times as illustrated in previous examples. This limitation can be mitigated to some extent by taking $k$-th nearest neighbor into consideration~\cite{yankov2008disk}, but how to adaptively select an appropriate value of $k$ still remains a challenge. 

\subsection{Anomaly Detection based on One-Class Classifier}
One-class approaches are commonly used in anomaly detection since the majority of the data is usually normal in anomaly detection. The One-Class Support-Vector-Machine (\cite{scholkopf2001estimating}, OC-SVM) and Support Vector Data Description (\cite{tax2004support}, SVDD) are two well-known methods in this category, which use a hyperplane and hypersphere to separate the normal data from anomalous data in the kernel-induced feature space, respectively. These methods can also be extended to time series anomaly detection~\cite{ma2003time}. However, their performances are often not satisfied when facing complex and high-dimensional datasets, due to the curse of dimensionality and non-efficient computational scalability. To overcome these limitations, Deep Support Vector Data Description (\cite{ruff2018deep}, DeepSVDD) is proposed to jointly train a deep neural network and optimize a data-enclosing hypersphere in output space. 
Inspired by Deep SVDD, THOC~\cite{shen2020timeseries} extends this deep one-class model by considering multiple spheres from all intermediate layers of a dilated recurrent neural network to extract multi-scale temporal features for better performance. 
Furthermore, Neural Contextual Anomaly Detection (\cite{ncad_ijcai22}, NCAD) incorporates the idea of Hypersphere Classifier (HSC) which improves the Deep SVDD by utilizing known anomalies data under the (semi-)supervised setting~\cite{ruff2020rethinking}, as well as a window-based approach specified for time series anomaly detection.
Another popular one-class classifier based method is the Deep Autoencoding Gaussian Mixture Model (\cite{zong2018deep}, DAGMM) which integrates a deep autoencoder with the Gaussian mixture model (GMM) to generate a low-dimensional representation to capture the normal data pattern. However, directly applying DAGMM in time series may not lead to improved outlier detection performance as it does not properly model temporal dependencies.  
 

\subsection{Anomaly Detection based on Reconstruction}
The reconstruction-based models usually learn a representation of the time series in latent space, and then reconstruct the signal from that representation. The anomalies are determined based on the reconstruction error. The long short-term memory based variational autoencoder (\cite{park2018multimodal}, LSTM-VAE) adopts serially connected LSTM and VAE layers to obtain the latent representation, estimates the expected distribution from the representation, and detects an anomaly when the log-likelihood of current observation given the expected distribution is lower than a threshold. 
OmniAnomaly~\cite{su2019robust} designs a stochastic recurrent neural network with state space models and normalizing flows for multivariate time series anomaly detection.
AnoGAN~\cite{schlegl2017unsupervised} proposes a deep generative adversarial network (GAN) to model data distribution and estimate their probabilities in latent space for anomaly detection. 
The Multi-Scale Convolutional Recursive Encoder Decoder (\cite{zhang2019deep}, MSCRED) uses a convolutional encoder and convolutional LSTM network to capture the inter-series and temporal patterns, respectively, and adopts a convolutional decoder to reconstruct the input time series for anomaly detection.

\subsection{Other Anomaly Detection Methods}
The rest of the anomaly detection methods roughly include density-based, transformer-based, and graph-based schemes. The Local Outlier Factor (\cite{breunig2000lof}, LOF) is a classic density-estimation method that assigns each object a degree of being an outlier depending on how isolated the object is with respect to the surrounding neighborhood. 
Anomaly Transformer~\cite{xu2021anomaly} proposes a new anomaly-attention mechanism to replace the original attention module and compute the association discrepancy, which can amplify the normal-abnormal distinguishability in time series under a minimax strategy to facilitate anomaly detection. 
Series2Graph~\cite{boniol13series2graph} transforms time series subsequences into a lower-dimensional space and constructs a directed cyclic graph, where the graph’s edges represent the transitions between groups of subsequences and can be utilized for anomaly detection. 
MTAD-GAT~\cite{zhao2020multivariate} designs two graph attention layers for learning the dependencies of time series in both temporal and feature dimensions, and then jointly optimizes a forecasting-based model and a reconstruction-based model for better anomaly detection results. While the existing graph-based methods achieve good results in some scenarios, they still have difficulty modeling and adapting to the challenging variable-length subsequences anomalies.

\section{Preliminaries}\label{subsec: problem defination}
\subsection{Notations and Problem Definition} 
A length $T$ univariate time series is denoted as $\boldsymbol{x} = (x_1, x_2, \cdots, x_T )$, where $x_t$ is a real-valued observation at time step $t$. A subsequence $\boldsymbol{x}_{t,L} = ( x_t, x_{t+1}, \cdots, x_{t+L-1})$ is a series of continuous observations from time series $\boldsymbol{x}$ starting at time step $t$ with length $L$, where $1 \leq t \leq T-L+1$. The problem of subsequence anomaly detection is to output the anomaly score $\mathrm{s}(\boldsymbol{x}_{t,L})$ for each subsequence which should be high for anomalous data and low for the normal. Then one can sort the scores in descending order to detect anomalies. 

Using the sliding window strategy with a stride $\tau$, a time series can be split into an ordered set of subsequences $\mathbf{X} \in \mathbb{R}^{N \times L}$, where $N$ is the number of subsequences and $L$ is the subsequence length. We denote the $i$-th subsequence $\boldsymbol{x}_{(i-1)\tau + 1,L}$ as $\mathbf{X}_i \in \mathbb{R}^{L}$. We assume that the initial maximum subsequence length $L$ is set to be large enough to support detecting various anomalies, which can be typically determined based on domain knowledge. For instance, $L$ can be four times of the period length for some periodic time series, e.g., the weather time series.

Moreover, we derive a multi-length view of subsequences with initial length $L$, and thus the proposed length selection mechanism can select a proper length. Specifically, the multi-length view of subsequence $\mathbf{X}_i$ is defined as $\mathcal{X}_i = \{ \mathbf{X}_{i,:l} \vert l = 2^p \Delta, p = 0,1,\cdots,P \}$, where $\mathbf{X}_{i,:l}$ indicates the first $l$ observations of $\mathbf{X}_i$ with $l$ increasing exponentially, $P$ is the largest positive integer satisfying $2^P \Delta \leq L$, and $\Delta$ is the length of an indivisible segment.

Note that in most real-world applications, finding and labeling anomalies are extremely expensive and time-consuming, while the labeled normal data are much easier to access. Hence, TSAD methods are usually trained with unlabeled data which are commonly treated as normal data as we assume that most data are normal due to the rareness of anomalies.

\subsection{Distance-based Anomaly Detection}
Distance-based anomaly detection methods usually calculate the distance between a target sample (or subsequence) and its reference in explicit data space or latent representation space as the anomaly score. We now provide details on a prevailing part of them.

\textbf{Time Series Discords.} \hspace{1mm}
Time series discords compute anomaly score of $\mathbf{X}_i$ as
\begin{equation}
    \mathrm{s}(\mathbf{X}_i) = \left\| \mathrm{z\mbox{-}norm}(\mathbf{X}_i) - \mathrm{z\mbox{-}norm}(\mathbf{X}_i^{(k-\mathrm{NN})})
    \right\| ^ 2,
\end{equation}
where $\mathrm{z\mbox{-}norm}(\mathbf{x}) = \left( \mathbf{x} - \mathrm{mean}(\mathbf{x}) \right) / {\mathrm{std}(\mathbf{x})}$ returns zero-normalized subsequence with $\mathrm{mean}(\cdot)$ and $\mathrm{std}(\cdot)$ standing for mean and standard deviation of input subsequence, and the corresponding reference $\mathbf{X}_i^{(k-\mathrm{NN})}$ is the $k$-th nearest neighbor\footnote{$k$ is usually set to $1$.} of $\mathbf{X}_i$. Despite its effectiveness and wide usage, it has several limitations as we discussed in Section \ref{sec:intro}.

\textbf{Deep Support Vector Data Description.} \hspace{1mm}
The SVDD and OC-SVM rely on a proper kernel to map data features to a high dimensional space for data separation. And the DeepSVDD algorithm replaces the kernel-induced feature space in the SVDD method with the feature space learned in a deep neural network. Specifically, DeepSVDD~\cite{ruff2018deep} is an unsupervised anomaly detection method that solves the following optimization problem
\begin{equation}
    \min \textstyle\frac{1}{N}\sum_{i=1}^{N}\|\mathrm{NN}(\mathbf{X}_i) - \mathbf{c}\| ^ 2,
\end{equation}
where DeepSVDD calculates the latent distance $\|\mathrm{NN}(\mathbf{X}_i) - \mathbf{c}\| ^ 2$ as anomaly scores, $\mathrm{NN}(\cdot)$ is a deep neural network, and $\mathbf{c}$ is a global reference which is the center of all the training data. 


\textbf{Hypersphere Classifier.} \hspace{1mm}
The Hypersphere Classifier (\cite{ruff2020rethinking}, HSC) extends DeepSVDD by training the network with the binary cross entropy objective function, which extends the model training with labeled anomalies. In particular, the HSC loss is given by
\begin{equation}\label{HSC_loss}
    -(1-y_i)\log l(\mathrm{NN}(\mathbf{X}_i)) - y_i\log(1- l(\mathrm{NN}(\mathbf{X}_i))), 
\end{equation}
where $y_i \in \{0,1\}$ with $0$ for the normal and $1$ for the anomalous, and $l:\mathbb{R}^d \rightarrow (0,1)$ maps the representation to a probabilistic prediction. Choosing $l(z) = \exp(-\|z\|^2)$ would reduce Equation (\ref{HSC_loss}) to DeepSVDD objective with center $\mathbf{c} = 0$ when all labels are $0$, i.e., all training samples are normal.

\subsection{Graph Neural Networks} \label{subsec:gnndef}
Graph neural networks (GNNs) ~\cite{GNN:survey:2021} operate on a graph $\mathcal{G} = (V, E, \mathbf{A})$, where  $V = [N] := \{ 1,\cdots,N \}$ is the set of node indices, $E$ is the set of edges, and $\mathbf{A} \in \mathbb{R}^{N \times N}$ is the weighted adjacency matrix representing the nodes' proximity with $\mathbf{A}_{ij}>0$ denoting a directed edge $(i,j)$ exists and $\mathbf{A}_{ij}=0$ otherwise.
Let $\mathbf{H} \in \mathbb{R}^{N \times d}$ be node representations, in which the $i$-th row $\mathbf{H}_i \in \mathbb{R}^d $ is the vectorized representation attached to node $i$. A GNN takes the node representations $\mathbf{H}$ along with the graph $\mathcal{G}$ as input and returns updated node representations $\mathbf{H}^{\prime} \in \mathbb{R}^{N \times d^{\prime}}$ using a message passing paradigm.
A simplified GNN used in this work is defined as follows:
\begin{equation}\label{node_representation}
    \mathbf{H}' = \rho(\mathbf{D}^{-1}\mathbf{A}\mathbf{H}\mathbf{W}_1 + \mathbf{H}\mathbf{W}_2 + \mathbf{b}), 
\end{equation}
where $\mathbf{W}_1, \mathbf{W}_2 \in \mathbb{R}^{d \times d'}$ are weight matrices that transform neighbor node features and target node features, respectively, $\mathbf{b} \in \mathbb{R}^{d'}$ is the bias vector, $\mathbf{D}  \in \mathbb{R}^{N \times N}$ is the diagonal degree matrix of $\mathbf{A}$, i.e., $\mathbf{D}_{ii} = \sum_j{\mathbf{A}_{ij}}$, $\mathbf{D}^{-1}\mathbf{A}$ represents the normalized adjacency matrix, and $\rho$ is a nonlinear activation function. A layer of GNN can aggregate information of 1-hop neighbors for each node. By stacking multiple layers, the receptive neighborhood range can be expanded.



\section{Methodology}
In this section, we describe GraphSubDetector, a graph neural network-based subsequence anomaly detection approach. Figure \ref{fig:model oveview} shows the proposed architecture.  A time series is first split into subsequences using sliding window. Then a multi-length feature encoder generates subsequence representations from a multi-length view and a length selection mechanism selects proper subsequence length. Moreover, subsequence proximities are into a prior graph structure using heuristics of time series discords, based on which, the message passing adjacency matrix is learned considering temporal information and distance between subsequences in both data and latent space, and then refine it using local density information. By performing graph neural network on the learned graph, the variance of representations of normal data can be reduced, while the discrepancy between normal and anomalous data is retained. 
Thus, the final generated representations can facilitate anomaly detection with superior performance.

\begin{figure*}
  \begin{minipage}[ht]{0.49\linewidth}
    \centering
    \includegraphics[width=1\textwidth]{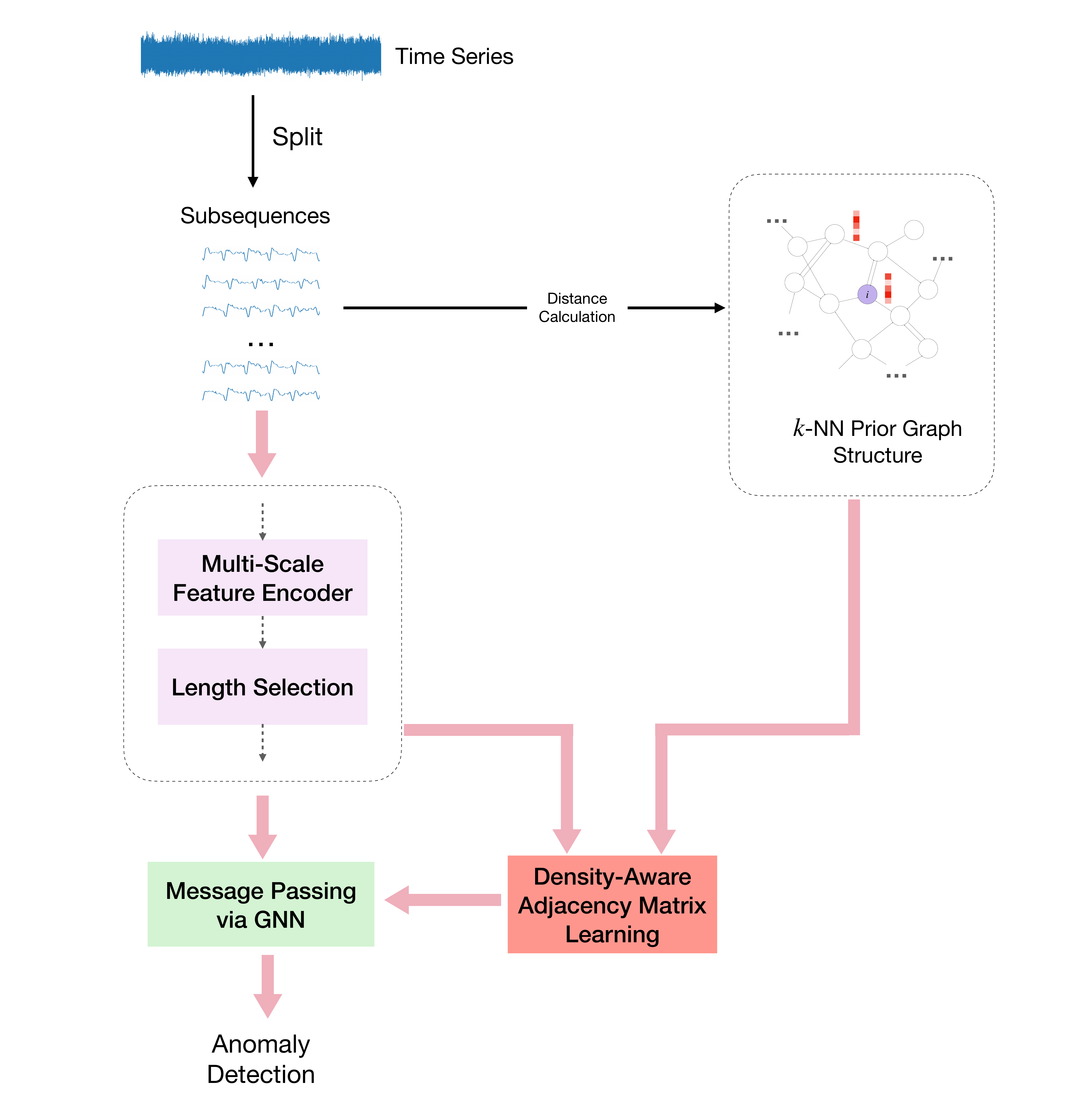}\vspace{-1mm}
    \caption{\small{An overview of the proposed GraphSubDetector.}}
    \label{fig:model oveview}
  \end{minipage}
  \hspace{2mm}
  \begin{minipage}[ht]{0.49\linewidth}
    \centering
    \includegraphics[width=1\textwidth]{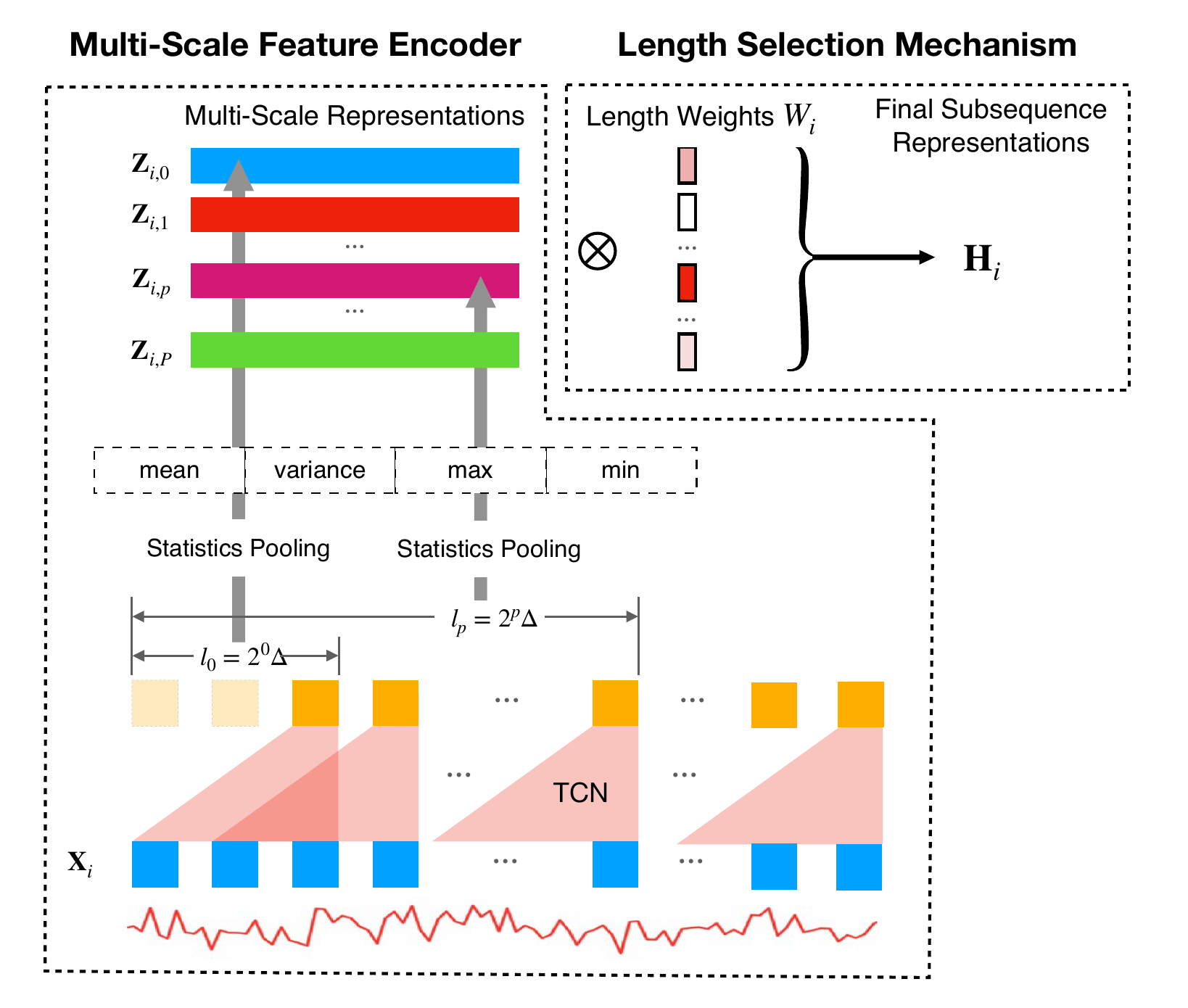}\vspace{-1mm}
    \caption{\small{Learning subsequence representations using multi-length encoder and length selection mechanism.}}
    \label{fig:multi-scale encode}
  \end{minipage}
\end{figure*}



\subsection{Learning Representations of Subsequences with Appropriate Lengths}

While selecting a proper subsequence length that highlights the characteristics of both normal and anomalous patterns is important for TSAD, most of the existing algorithms commonly set the subsequence length through trial and error or by some prior knowledge (e.g., the period length for a periodic time series). 
A potential modification is to perform anomaly detection with multiple subsequence lengths and vote for anomalies, while it is still not robust since anomaly scores of different lengths are hard to align. Thus, in order to learn expressive subsequence representations with proper length, we propose a temporal convolution network (TCN)~\cite{bai2018empirical} based multi-length feature encoder aggregating information at different length resolutions as well as a learnable length selection mechanism to generate representations of proper length.

\textbf{Multi-length Feature Encoder.} \hspace{1mm}
This encoder takes the raw subsequence $\mathbf{X}_i \in \mathbb{R} ^ L$ as input, and output the multi-length representations $\mathcal{Z}_i = \{ \mathbf{Z}_{i,0}, \cdots, \mathbf{Z}_{i,P} \}$, where $P$ is the number of length scales, defined in Section \ref{subsec: problem defination}. 
First, the TCN consisted of several layers of causal convolution with $\mathrm{ReLU}$ activation and layer normalization outputs the intermediate embeddings $\mathbf{R}_i = \mathrm{TCN}(\mathbf{X}_i) \in \mathbb{R}^{L \times d}$, where $d$ is the hidden dimension and necessary padding is performed to keep the subsequence length unchanged. Then, we aggregate information of variable length $l$ by pooling $\mathbf{R}_{i,:l} \in \mathbb{R}^{l \times d}$ into a single vector. Despite the fact that mean-pooling or max-pooling is frequently used to generate a single vector summarizing a sequence, they fail to characterize the statistical property adequately which is crucial for anomaly detection. For example, the maximum and minimum information are helpful for detecting spike and dip anomalies, and the variance increases when noises occur. To integrate this inductive bias into representation space, we define a statistics pooling operator $\mathrm{StatsPool}(\cdot)$ that calculates and concatenates basic statistics including mean, variance, maximum, and minimum of an $l$-length subsequence embedding through the first dimension (time axis) as:
\begin{align} 
   \mathbf{Z}_{i,p} \! &= \! \mathrm{StatsPool}(\mathbf{R}_{i,:l}) \notag \\ 
   \!&=\! \left[ \mathrm{mean}(\mathbf{R}_{i,:l});
   \mathrm{var}(\mathbf{R}_{i,:l});
   \mathrm{max}(\mathbf{R}_{i,:l});
   \mathrm{min}(\mathbf{R}_{i,:l}) \right],
\end{align}
where length $l=2^p \Delta$ is the $p$-th scale length. By performing the statistics pooling with multi-length view from scale $0$ to $P$, we can generate the multi-length representations $\mathcal{Z}_i = \{ \mathbf{Z}_{i,0}, \cdots, \mathbf{Z}_{i,P} \}$ with $\mathbf{Z}_{i,p} \in \mathbb{R} ^ {4d}$.

\textbf{Length Selection Mechanism.} \hspace{1mm}
As multi-length subsequence representations are generated, the next problem is how to select proper lengths among them. We propose a simple yet effective multi-length selection mechanism to learn an embedding $\mathbf{W}_i \in \mathbb{R} ^ P$ for each subsequence. Then the subsequence representation $\mathbf{H}_{i}$ can be generated as follows: 
\begin{equation}
\begin{split}
   \mathbf{Z}_{i} &= \mathrm{Softmax}( \mathbf{W}_i ) (\mathbf{Z}_{i,0}, \cdots, \mathbf{Z}_{i,P})^{\top} \in \mathbb{R} ^ {4d},   \\
   \mathbf{H}_{i} &= \mathrm{MLP}(\mathbf{Z}_{i}) \in \mathbb{R} ^ {d}.
\end{split}
\end{equation}

The length selection embedding $\mathbf{W}_i$ is initialized to an all-zero vector for anomaly-agnostic scenarios, which is then jointly learned with model parameters (in Section \ref{model learning}). As a result, the weights of different lengths are the same in the beginning, and during the training procedure, this embedding is optimized to pay more attention to the proper length that minimizes the anomaly detection loss function. It is worth mentioning that one can introduce priors of subsequence length by adjusting the initialization of $\mathbf{W}_i$ to $\boldsymbol{e}^{(p)} = [0,\dots,1,\dots,0] \in \mathbb{R} ^ P$ with $1$ at position $p$. For instance, $\boldsymbol{e}^{(0)}$ can be set for the time series dataset in which point-wise anomalies (e.g., spike and dip) often occur.

\subsection{Robust Representation Learning via Density-Aware Graph Neural Networks} \label{subsec:gnn}

As aforementioned, the inherent variance and noise in normal data always prevent detectors from effectively and robustly distinguishing anomalies from normal patterns, hence we propose to alleviate it by performing message passing between subsequences as noise removal. Below, we theoretically verify that, with properly-designed adjacency matrix, the discrepancy between anomalous and normal data can be retained while the variance of normal data is reduced, which makes anomaly prominent in data distribution.

Let $N, M$ denote the number of normal and anomalous data samples, respectively, $N \gg M$, and $d$ denotes the dimension of the features. Assume that normal samples $\mathbf{f}_1^{\mathrm{nor}}, \cdots, \mathbf{f}_N^{\mathrm{nor}} \sim \mathcal{N}(\boldsymbol{\mu}, \sigma^2)$ and anomalous samples deviate from normal ones. The generation procedure of anomalous samples is $\mathbf{f}_i^{\mathrm{ano}} = \mathbf{f}_i^{\mathrm{ref}} + \boldsymbol{\epsilon}_i$, where $1 \leq i \leq M$, $\mathbf{f}_i^{\mathrm{ref}} \sim \mathcal{N}(\boldsymbol{\mu}, \sigma^2)$ is the reference normal sample of the $i$-th anomaly, $\| \boldsymbol{\epsilon}_i \| = K \sigma$ is the corresponding nontrivial anomalous deviation, and $K > 0$. 

\begin{theorem} \label{theorem_1}
\label{thm:bigtheorem}
 Define a fully-connected graph $\mathcal{G}$ of data samples with adjacency matrix $\mathbf{A} \in \mathbb{R} ^ {(N+M) \times (N+M)}$ encoding similarity measurement of pair-wise samples, and the entry $\mathbf{A}_{ij} = \exp{(-\|\mathbf{f}_i - \mathbf{f}_j \| ^ 2 / \delta})$, which transforms Euclidean distance between features with a Gaussian kernel. Denote $\mathbf{F} \in \mathbb{R} ^{(N+M) \times d}$ as the feature matrix of all samples, $\mathbf{F}_{\star \mathcal{G}} = \mathbf{D}^{-1}\mathbf{A}\mathbf{F}$ is the feature matrix after message passing w.r.t. graph $\mathcal{G}$, and $\sigma_{\star \mathcal{G}}$ is the standard deviation of the newly generated features. 
 Then for any $K > 0$, there exist  $\delta > 0$ that for each anomalous sample we have 
 $
     {\| \mathbf{f}_{i, \star \mathcal{G}}^{\mathrm{ano}} - \mathbf{f}_{i, \star \mathcal{G}}^{\mathrm{ref}} \|}/{\sigma_{\star \mathcal{G}}} > {\|\mathbf{f}_i^{\mathrm{ano}} - \mathbf{f}_i^{\mathrm{ref}} \|}/{\sigma}.
 $
\end{theorem}

Proof and detailed analysis are left in Appendix \ref{appendix:theory}. Theorem \ref{theorem_1} can be concluded that by performing message passing on a similarity graph as defined above, the discrepancy between anomalous and normal data can be enlarged relatively through the variance reduction of normal data. Moreover, We find that by slightly modifying the formulation of the above adjacency matrix, the relative discrepancy of transformed anomalous and normal features can be further enlarged.

\begin{theorem} \label{theorem_2}
\label{thm:bigtheorem}
 Define fully connected graph $\mathcal{\hat{G}}$ with adjacency matrix $\mathbf{\hat{A}}$, and $\mathbf{\hat{A}}_{ij} = \exp{(-\|\mathbf{f}_i - \mathbf{f}_j \| ^ 2 / \delta} - \| \mathbf{f}_i - \boldsymbol{\mu} \| ^ 2 / c ) $. $\mathbf{F}_{\star \mathcal{\hat{G}}} = \mathbf{\hat{D}}^{-1}\mathbf{\hat{A}}\mathbf{F}$ is the feature matrix after message passing w.r.t. graph $\mathcal{\hat{G}}$, and $\sigma_{\star \mathcal{\hat{G}}}$ is the standard deviation of the newly generated features. 
 Then for any $\delta > 0$, there exist  $c > 0$ that for each anomalous sample we have
$
     {\| \mathbf{f}_{i, \star \mathcal{\hat{G}}}^{\mathrm{ano}} - \mathbf{f}_{i, \star \mathcal{\hat{G}}}^{\mathrm{ref}} \|}/{\sigma_{\star \mathcal{\hat{G}}}} > {\| \mathbf{f}_{i, \star \mathcal{G}}^{\mathrm{ano}} - \mathbf{f}_{i, \star \mathcal{G}}^{\mathrm{ref}} \|}/{\sigma_{\star \mathcal{G}}}.
$
\end{theorem}

The detailed proof is left in Appendix \ref{appendix:theory}. Note that the second part $\exp{(- \| \mathbf{f}_j - \boldsymbol{\mu} \| ^ 2 / c )}$ in similarity measurement is actually a transformation of the probability density of source node $\mathbf{f}_j $ w.r.t. the normal sample distribution $\mathcal{N}(\boldsymbol{\mu}, \sigma^2)$, which means introducing a probability density-dependent factor is helpful to enlarge the relative discrepancy between anomalous and normal data compared with normal data variance.

\textbf{Density-Aware Adaptive Graph Neural Networks.} \hspace{1mm} 
The above analysis implies that incorporating data similarities and density information to learn adjacency matrix is helpful for learning robust representations against the inherent variance of normal data. Based on this motivation, we design a novel density-aware adaptive graph neural network (DAGNN), consisting of the following three components:

\emph{$K$-nearest neighbor prior graph.} \hspace{1mm} 
Motivated by time series discords, we encode distance information in explicit data space as the prior graph.
We build the prior graph following three principles: (1) multiple distance measures in data space should be considered to characterize subsequence relationships from different perspectives, (2) distance information of subsequence with multiple lengths should be introduced, and (3) only informative neighbors of each node are reserved for the sparsity and purity of the graph. 

The prior graph is described by a 3-tuple: $\mathcal{G}^{\mathrm{prior}}  \!\!=\! (V, E^ {\mathrm{prior}}, \mathbf{A}^ {\mathrm{prior}}) $, accompanied with edge attributes $ \mathbf{E} $, where  $V = [N]$ is the set of node indices representing a subsequence each, $\mathbf{A}^ {\mathrm{prior}}  \in \{0,1\}^{N \times N}$ is the initial unweighted adjacency matrix, and edge attributes $ \mathbf{E}  \in \mathbb{R}^{N \times N \times d_{e}}$ store the distance information between subsequences. 
We take the $\mathbf{A}^ {\mathrm{prior}}_{ij} = 1$ (a directed edge from $j$ to $i$) and the corresponding edge attribute $\mathbf{E}_{ij}  \in \mathbb{R}^{d_e}$ as an example. Recalling the multi-length view of subsequences  $\mathcal{X}_i = \{ \mathbf{X}_{i,:l} \vert l = 2^p \Delta, p = 0,1,\cdots,P \}$. We calculate Euclidean distance $d_{ij,l}$ and z-normalized Euclidean distance $d_{ij,l}^\mathrm{z\mbox{-}norm}$ between subsequences $i$ and $j$ of length  $l = 2^p \Delta$, $p = 0,1,\cdots,P$, and each distance $d_{ij}$ is an entry of the edge attribute $\mathbf{E}_{j,i} $. Moreover, the initial graph structure follows the definition of the $K$-NN graph to enforce sparsity, i.e., edge $(j,i)$ is present if any distance between $i$ and $j$ is among the $K$-th smallest distances between $i$ and other nodes.

\emph{Density-aware adaptive adjacency matrix learning.} \hspace{1mm}
During adjacency matrix learning in practice, we go beyond Theorem \ref{theorem_1} by not only using the distance between (latent) features to generate adjacency matrix but also considering distance information in data space as well as temporal information. We do not change the graph structure of the initial graph for efficient computation, and the adaptive adjacency matrix $\mathbf{A}$ is calculated as:
\begin{equation} \notag
\resizebox{1.0\hsize}{!}{$
    \mathbf{A}_{ij} = \left\{
    \begin{array}{cc}
    e^{-\frac{\| \mathbf{H}_i - \mathbf{H}_j \|^2}{\delta_1}-\frac{ \mathrm{MLP}(\mathbf{E}_{ij}) }{\delta_2}-\frac{| \mathrm{pos}_i - \mathrm{pos}_j |~ \mathrm{mod} ~T}{\delta_3}},
     & \mathbf{A}^ {\mathrm{prior}}_{ij} = 1, \\
     0, &  \mathbf{A}^ {\mathrm{prior}}_{ij} = 0,
    \end{array}
    \right.
$}
\end{equation}
where $\delta_{1,2,3}$ are hyperparameters, $\mathbf{H}$ is the representations from upstream neural networks, $\mathrm{MLP}(\cdot)$ stands for a certain multi-layer perception network, and $\mathrm{MLP}(\mathbf{E}_{ij})$ utilizes the multi-length distance information in data space. $\mathrm{pos}_i$ and $ \mathrm{pos}_j$ are the position of subsequence $\mathbf{X}_i$ and $\mathbf{X}_j$, respectively, and $T$ is the period length. $| \mathrm{pos}_i - \mathrm{pos}_j |~ \mathrm{mod}~ T$ accounts for periodic temporal distance, and for non-periodic series, this term will be omitted.

According to Theorem \ref{theorem_2}, we further need to refine the learned adjacency matrix with density information, while the density of each node cannot be directly estimated. We find that the density of a node can be revealed using its similarity to neighbor nodes. A toy example is that the node with high similarity to its neighbors always has a higher density, and vice versa. As a result, we generate density-aware adjacency matrix $\hat{\mathbf{A}}$ on the top of $\mathbf{A}$ as:
\begin{equation}
    \hat{\mathbf{A}}_{ij} = \left\{
    \begin{array}{cc}
     \mathbf{A}_{ij} \exp\{-\frac{ \mathrm{MLP}(\mathbf{A}_{i:}) }{\delta_4} \},  
     & \mathbf{A}^ {\mathrm{prior}}_{ij} = 1,\\
     0, &  \mathbf{A}^ {\mathrm{prior}}_{ij} = 0,
    \end{array}
    \right.
\end{equation}
where $\mathbf{A}_{i:}$ is the $i$-th row vector of $\mathbf{A}$ which gives the profile of node $i$'s similarity to its neighbor nodes implying density information.

\emph{Message passing.} \hspace{1mm}
Once the density-aware adjacency matrix is learned, the message-passing procedure is performed as:
\begin{equation}\label{node_representation}
    \mathbf{H}' = \rho(\hat{\mathbf{D}}^{-1} \hat{\mathbf{A}}\mathbf{H}\mathbf{W}_1 + \mathbf{H}\mathbf{W}_2 + \mathbf{b}), 
\end{equation}
where the notations are defined in Section \ref{subsec:gnndef}.

\subsection{Model Learning} \label{model learning}

\textbf{Anomaly Injection.} \hspace{1mm} 
As we mentioned above, the labeled anomaly data are rare for training in most situations. Here the anomaly injection method is used to inject several labeled artificial anomalous subsequences into the time series. We consider six different types of anomalies, including spike \& dip, resizing, warping, noise injection, left-side right, and upside down. In more detail, spike \& dip stands for standard point anomaly with unreasonable high or low value, resizing randomly changes the temporal resolution of subsequence, warping denotes random time warping and noise injection for random noise injected, while the latter two stand for reversement of subsequence from left to right and up to down, respectively. Despite the fact that such artificial anomalies may not coincide with real anomalies, it is beneficial for learning characteristic representations of normal behaviors.

\textbf{Nearest Neighbour Contextual Anomaly Detection.} \hspace{1mm}
Following the distance-based anomaly detection methods, our approach identifies anomalies by comparing the target representation with its neighboring representations, and the anomaly score is defined as: 
\begin{equation} \label{equ:anomaly score}
    s(\mathbf{X}_i) = \frac{1}{N}\sum_{j =1}^{N} \mathbf{A}^{\mathrm{prior}}_{ij} \left\| \mathbf{H}^{\prime}_i - \mathbf{H}^{\prime}_j \right\|^2, 
\end{equation}
where $\mathbf{H}^{\prime}_i$ and $\mathbf{H}^{\prime}_j$ are node representations generated from DAGNN. $s(\mathbf{X}_i)$ returns the mean of Euclidean distances between the target and its neighbors as anomaly scores. By incorporating artificially labeled anomalies, we can derive a loss function $\mathcal{L}$ following the Hypersphere Classifier (HSC)~\cite{ruff2020rethinking} objective:
\begin{equation}
    \mathcal{L} = \textstyle\frac{1}{N} \sum_{i=1}^{N} -(1 - y_i) s(\mathbf{X}_i) - y_i\log (1 - \exp(-s(\mathbf{X}_i))),
\end{equation}
where $s(\mathbf{X}_i)$ is the anomaly score of subsequence $\mathbf{X}_i$ (see Equation (\ref{equ:anomaly score})), and $y_i \in \{ 0,1 \}$ with $1$ for (artificial) anomalous subsequences and $0$ for normal and unlabeled subsequences. 
Furthermore, in the scenarios with both normal and unlabeled data in the training set, we actually use a weighted version of HSC objective, where normal data is assigned with a larger weight than unlabeled data, as if a known normal sample gets a high anomaly score, it needs to be penalized more.

\textbf{Regularizations.} \hspace{1mm} We introduce two regularizations for stable learning, including auto-encoding regularization and length selection embedding regularization. 

\emph {Auto-encoding regularization.} We introduce an auxiliary MLP decoder to reconstruct the original subsequence $\mathbf{X}_i$ using the output of GNN  $\mathbf{H}^{\prime}_i$. The auto-encoding regularization is defined as
\begin{equation}
    \mathcal{L}_{\mathrm{dec}} = \textstyle\frac{1}{N} \sum_{i=1}^{N} \mathrm{MSE}(\mathbf{X}_i, \mathrm{Dec}(\mathbf{H}^{\prime}_i)),
\end{equation}
where $\mathrm{MSE}(\cdot)$ denotes the mean squared error loss and $\mathrm{Dec}(\cdot)$ is the decoder. Suppose that artificial anomalies are not involved, i.e., $y_i = 0$ for all training samples, the network might corrupt to map all samples to a constant representation and achieve the optimal. This regularization is utilized to restrict the network to preserve the information of input data to avoid trivial solutions. Even with anomaly injection, it can still conduce to the model's robustness.

\emph {Length selection embedding regularization.} We impose a restriction that the length selection embeddings of adjacent nodes should be close, which is defined as
\begin{equation}
    \mathcal{L}_{\mathrm{len}} = \frac{1}{|E|}\sum_{\mathbf{A}^{\mathrm{prior}}_{ij} = 1 } \left\|  \mathbf{W}_i - \mathbf{W}_j \right\| ^ 2,
\end{equation}

where $\mathbf{W}\in \mathbb{R}^{N\times P}$, 
the row vector $\mathbf{W}_i\in \mathbb{R}^P, \forall i\in \{1,\cdots,N\}$ is the length selection embedding of each node, 
and $\|E\|$ is the cardinality of the edge set, i.e., the total number of edges.

\textbf{Training Procedure.} \hspace{1mm} As two components need to be optimized, i.e., model parameters $\boldsymbol{\theta}$ and length selection embedding $\mathbf{W}$, we propose a training strategy with two phases operating alternately towards better control of joint optimization
\begin{equation}\label{eq_training}
\begin{aligned}
    \boldsymbol{\theta} \leftarrow \mathrm{arg} \min_{\boldsymbol{\theta}} \mathcal{L} + \lambda \mathcal{L}_{\mathrm{dec}}, ~~~
   \mathbf{W} \leftarrow \mathrm{arg} \min_{\mathbf{W}} \mathcal{L} + \mu \mathcal{L}_{\mathrm{len}},
\end{aligned}
\end{equation}
where $\lambda > 0$ and $\mu > 0$ are hyperparameters. The two phases above alternate during model training.


\section{Experiments}

In this section, we compare the performance of our approach with other methods on multiple benchmark datasets, conduct case studies to analyze the model's behavior, evaluate the sensitivity of model parameters, and investigate model variations in ablation studies. 
\subsection{Datasets and Evaluation Metrics}

\textbf{Datasets.} \hspace{1mm} We adopt the following eight annotated real and synthetic datasets from various scenarios to conduct comprehensive evaluations.
1)
$\mathrm{UCR}$\footnote{\url{https://www.cs.ucr.edu/~eamonn/time_series_data_2018/}}: the well-known subsequence anomaly dataset from ``KDD cup 2021 multi-dataset time series anomaly detection'' competition, consisting of 250 univariate time series from different domains with one subsequence anomaly per series. The lengths of the series vary from 4000 to one million.
2)
$\mathrm{UCR\mbox{-}Aug}\footnote{This dataset is included in the Anonymous GitHub link in the Abstract.}$: since $\mathrm{UCR}$ contains one anomaly per time series, we augment $\mathrm{UCR}$ by adding various types of subsequence anomalies with variable-length into each time series which is more consistent with most real-world scenarios. 
3)
$\mathrm{SED}$~\cite{aggarwal2017introduction}: from the NASA Rotary Dynamics Laboratory, records disk revolutions measured over several runs (3K rpm speed).
4)
$\mathrm{IOPS}$\footnote{\url{http://iops.ai/dataset_detail/?id=10}}~\cite{paparrizos2022tsb}: is a dataset with performance indicators that reflect the scale, quality of web services, and health status of a machine.
5)
$\mathrm{ECG}$~\cite{moody2001impact}: is a standard electrocardiogram dataset and the anomalies represent ventricular premature contractions. Long series MBA (14046) is split to 47 series.
6)
$\mathrm{SMAP}$ and 7) $\mathrm{MSL}$\footnote{$\mathrm{SMAP}$, $\mathrm{MSL}$ and the following $\mathrm{SMD}$ have a predefined train/test split, we do not use labels in training set following unsupervised paradigm.}: Soil Moisture Active Passive satellite (SMAP) and Mars Science Laboratory rover (MSL), two datasets published by NASA \cite{hundman2018detecting}, with 55 and 27 series, respectively. The lengths of the time series vary from 300 to 8500 observations.
8) 
$\mathrm{SMD}$: Server Machine Dataset~\cite{su2019robust}, a 5 weeks long dataset with 38-dimensional time series each collected from a different machine in large internet companies. 

\textbf{Evaluation Metrics.} \hspace{1mm} For the first five subsequence anomaly datasets, we choose AUC, VUS~\cite{paparrizos2022volume} and Recall@$k$ as evaluation metrics. AUC stands for the area under receiver operating characteristic (ROC) curves. VUS~\cite{paparrizos2022volume} is a recently proposed new metric better for subsequence anomaly detection, which stands for volume under the surface of ROC. Recall@$k$ is the recall rate for the number of anomalous subsequences found in the top-$kn$ anomaly scored ones where $n$ is the total number of anomalous subsequences in one time series, i.e., for each anomaly, we have $k$ reporting opportunities considering the cost of check is limited in reality. For the last three multivariate datasets, we report F1 scores computed by choosing the best threshold on the test set following prior works~\cite{su2019robust,shen2020timeseries,ncad_ijcai22}.
\subsection{Implementation Details}
We set the length of an indivisible segment $\Delta=0.125m$ for periodic time series with period length $m$, and $\Delta=10$ for non-periodic time series. We utilize a sliding window with a stride $\tau = 2\Delta$ to generate subsequences. For the multi-length view of subsequences, we set the maximum length scale $P=5$, and thus the length can vary from 0.125 to 4 times of period length for periodic time series, which can stand for most common anomalies. 



\textbf{Hyperparameters.} \hspace{1mm} For hyperparameter tuning, we randomly select 8 time series to perform grid search. We use this inferred set of hyperparameters for all datasets. The hyperparameter settings are summarized in Table \ref{table:hyperparam}.

\begin{table*}[h]
\centering
\caption{Hyperparamter settings of the GraphSubDetector.} \label{table:hyperparam}
\scalebox{0.9}{
 \begin{tabular}{c| c | c }
    \toprule
        Hyperparameter & Description & Value \\
    \midrule
        $\Delta$ & the length of an indivisible segment  & \makecell[c]{$0.125T$ for periodic time series\\(where $T$ is the period length), and $10$ otherwise} \\
    \midrule
        $\tau$ & stride length & $2\Delta$ \\
    \midrule
        $P$ & the maximum length scale & $5$ \\
    \midrule
        $L$ & the maximum subsequence length & $2^P\Delta$ \\
    \midrule

        $\delta_{1,2,3,4}$ & \makecell[c]{the scale factors of multiple \\ similarities} & \makecell[c]{$\delta_1 = d$ (where $d$ is the hidden dimension), \\ $\delta_2 = 1.0$, \\ $\delta_3 = T$ (where $T$ is the period length), \\$\delta_4 = 1.0$ }  \\
    \midrule
        $\lambda$ & \makecell[c]{the coefficient of auto-encoding \\ regularization} & 1.0 \\
    \midrule
        $\mu$ & \makecell[c]{the coefficient of length selection\\ embedding regularization} & 0.2 \\
    \bottomrule
    \end{tabular}
}
\end{table*}


\textbf{Model Training.} \hspace{1mm} We run the model 5 times on each benchmark dataset and report the average score. We train the model with a Tesla V100 32GB GPU. For all datasets, we train GraphSubDetector for 10 epochs with a learning rate of $1 \times 10^{-4}$ for model parameters and  $5 \times 10^{-4}$ for length selection embeddings.


\subsection{Experiment Comparisons}\label{sec_expe}

\begin{table}[t]
\centering
\caption{Performance comparisons of different algorithms on $\mathrm{UCR,UCR\mbox{-}Aug}$ datasets under Recall@k metric. R indicates Recall. The best ones are in bold, and the second ones are {\underline{underlined}}.}
\scalebox{0.95}{
 \begin{tabular}{c|c| cccc }
    \toprule
    Dataset&  Algorithm  & R@1 & R@3 & R@5 & R@10   \\
    \midrule
   \multirow{5}{*}{$\mathrm{UCR}$} & Matrix Profile & \underline{0.436} & \underline{0.548}  & \underline{0.596} & 0.644 \\
     &  Series2Graph &  0.276  & 0.348  & 0.372  & 0.408 \\
     &  LOF  & 0.312 & 0.360 & 0.409  & 0.461 \\
     &  NCAD  & 0.352 & 0.412 & 0.520 & \underline{0.648} \\
    \cline{2-6}
     &  GraphSubDetector & \textbf{0.560}& \textbf{0.598} & \textbf{0.663} & \textbf{0.738} \\
    \midrule
   \multirow{5}{*}{$\mathrm{UCR\mbox{-}Aug}$}  & Matrix Profile  & 0.709 & \underline{0.841} & \underline{0.873} & \underline{0.899} \\
     & Series2Graph & 0.442 & 0.610 & 0.687 & 0.761 \\
     & LOF  & 0.594 & 0.736 & 0.798 & 0.861 \\
     & NCAD & \underline{0.722} & 0.803 & 0.839 & 0.890 \\
    \cline{2-6}
     & GraphSubDetector & \textbf{0.755} & \textbf{0.902} & \textbf{0.928} & \textbf{0.989} \\
    \bottomrule
    \end{tabular}\label{table:ucr-aug}
}
\end{table}

\begin{table*}[t]
\centering
\caption{Performance of models on subsequence anomaly datasets under AUC and VUS metrics. The best ones are in bold, and the second ones are {\underline{underlined}}.}
\scalebox{0.9}{
 \begin{tabular}{c| cc|cc|cc |cc|cc }
    \toprule
       Dataset & \multicolumn{2}{|c|}{$\mathrm{UCR}$} &\multicolumn{2}{|c|}{$\mathrm{UCR\mbox{-}Aug}$} &\multicolumn{2}{|c|}{$\mathrm{SED}$} &\multicolumn{2}{|c}{$\mathrm{IOPS}$}&\multicolumn{2}{|c}{$\mathrm{ECG}$}\\
    \midrule
      Metric &AUC& VUS & AUC & VUS   &AUC& VUS & AUC & VUS & AUC & VUS   \\
    \midrule
    Matrix Profile & 0.872 & 0.868 & 0.908 & 0.891 & \underline{0.996} & \underline{0.993} & 0.374 &  0.374& 0.796 & 0.753 \\
      PCA  & 0.568 & 0.561 & 0.571 & 0.562& 0.587 & 0.537 & 0.756 & 0.754 & 0.737 & 0.728 \\
      AE  & 0.770 & 0.763 & 0.861 & 0.848 & 0.962 & 0.951 & 0.596 & 0.584 & 0.686 & 0.679 \\
      NORMA & \underline{0.874} & \underline{0.872} & \underline{0.942} & \underline{0.935} & \textbf{0.998} & \textbf{0.994} & \underline{0.982} & \textbf{0.978} & 0.576& 0.569  \\
      IForest & 0.672 & 0.668 & 0.788 & 0.775& 0.669  & 0.650 & 0.976 & 0.955  & \underline{0.843}  & \underline{0.829} \\
      LSTM-AD & 0.676 & 0.672 & 0.767 & 0.760& 0.712 & 0.694 & 0.507 & 0.494 & 0.651 & 0.649\\
    \midrule
      GraphSubDetector & \textbf{0.913} & \textbf{0.907} & \textbf{0.956} & \textbf{0.945}& 0.995 & 0.991 & \textbf{0.985} & \underline{0.967} & \textbf{0.923} & \textbf{0.919}\\
    \bottomrule
    \end{tabular}\label{table:public_dats}
}
\end{table*}

We first compare GraphSubDetector with anomaly detectors that are specified for subsequence anomaly detection using subsequence anomaly datasets, including Matrix Profile~\cite{yeh2016matrix} (a time series discord search method), Series2Graph~\cite{boniol13series2graph}, LOF~\cite{breunig2000lof}, and NCAD~\cite{ncad_ijcai22}. 
Table \ref{table:ucr-aug} reports the performance of different algorithms on $\mathrm{UCR}$ and $\mathrm{UCR\mbox{-}Aug}$ datasets. It can be seen that our GraphSubDetector outperforms Matrix Profile (the second best method) by a reasonable margin on $\mathrm{UCR}$ dataset and a large margin on \textbf{$\mathrm{UCR\mbox{-}Aug}$} dataset. This is mainly because $\mathrm{UCR}$ has one anomaly per time series, to which time series discords are naturally applicable. While each time series in $\mathrm{UCR\mbox{-}Aug}$ is more challenging with multiple anomalies and some of them are similar in anomalous patterns. Furthermore, we compare GraphSubDetector with more classical anomaly detection baselines from different perspectives, including principal component analysis (PCA), AutoEncoder~\cite{sakurada2014anomaly}, NORMA~\cite{boniol2021unsupervised}, IForest and LSTM-AD~\cite{malhotra2015long}, on total five subsequence anomaly datasets. As shown in Table~\ref{table:public_dats}, we use AUC and VUS as they do not rely on a user-defined score threshold to predict normal or anomalous labels, especially VUS would give a more robust evaluation on subsequence anomalies. We observe that GraphSubDetector gives the best performance on most datasets. On the $\mathrm{UCR}$,$\mathrm{UCR\mbox{-}Aug}$, and $\mathrm{ECG}$ datasets, GraphSubDetector significantly outperforms all other baselines.

\begin{table}[h]
\centering
\caption{F1 score of models on multivariate datasets. The best ones are in bold, and the second ones are {\underline{underlined} }.}
\vspace{1mm}
\scalebox{0.9}{
 \begin{tabular}{c| c c c }
    \toprule
        Dataset & $\mathrm{SMAP}$ & $\mathrm{MSL}$ & $\mathrm{SMD}$ \\
    \midrule
      Deep SVDD & 71.71& 88.12 &--\\
        OmniAnomaly & 84.34 & 89.89 & \underline{88.57} \\
        THOC & \underline{95.18} & 93.67 & --\\
        NCAD  & 94.45 & \textbf{95.6}  &  80.16   \\
        GTA &90.41 & 91.11 & -- \\
        GDN &95.14 & 92.62 & --\\
    \midrule
        GraphSubDetector  & \textbf{96.67} & \underline{95.32}  & \textbf{91.12} \\
    \bottomrule
    \end{tabular}\label{table_multivar}
}
\end{table}

We also compare GraphSubDetector against the state-of-the-art deep learning models, including DeepSVDD~\cite{ruff2018deep}, OmniAnomaly~\cite{su2019robust}, THOC~\cite{shen2020timeseries}, NCAD~\cite{ncad_ijcai22}, GTA~\cite{chen2021learning}, and GDN~\cite{deng2021graph}, on SMAP, MSL, and SMD datasets (mainly containing point-wise anomalies) in Table \ref{table_multivar}. Note that those deep learning methods are optimized for detecting point-wise anomalies, which is not our main concern. However, GraphSubDetector still remains highly competitive. GraphSubDetector only has a slightly lower F1 score than NCAD on one dataset ($\mathrm{MSL}$), while it outperforms all baselines on the other two datasets (SMAP and SMD).





\begin{table}[h]
\setlength{\belowdisplayskip}{-2mm} 
\centering
\caption{Performance comparisons of model variations of GraphSubDetector on $\mathrm{UCR}$ and $\mathrm{UCR\mbox{-}Aug}$ datasets. R indicates Recall. The best results are highlighted in bold.}
\scalebox{0.8}{
 \begin{tabular}{c| ccc | ccc }
    \toprule
        Dataset & \multicolumn{3}{|c|}{$\mathrm{UCR}$} &\multicolumn{3}{|c}{$\mathrm{UCR\mbox{-}Aug}$}\\
    \midrule
      Metric  &AUC& R@1  & R@5  & AUC& R@1  & R@5  \\
    \midrule
      w/o graph & 0.7498 & 0.366 & 0.609 & 0.8713 & 0.714 & 0.881 \\
      w/o adaptive graph & 0.8632 & 0.503  & 0.625 & 0.9325 &0.720 &0.911 \\
      w/o density information & 0.8832 & 0.556 & \textbf{0.687} & 0.9291 &\textbf{0.766} &0.923 \\
      w/o length selection &  0.8939 & 0.478 & 0.620 & 0.9331  & 0.762 & 0.916 \\
      fixed length & 0.8965 &  0.487& 0.654 & 0.9247 & 0.741  & 0.900\\
    \midrule
      GraphSubDetector &\textbf{0.9134} & \textbf{0.560} & 0.663 & \textbf{0.9562} &0.755 & \textbf{0.928} \\
    \bottomrule
    \end{tabular}\label{table:ablation}
}
\end{table}

\subsection{Ablation Studies}

To better understand the effectiveness of each component in GraphSubDetector, we perform ablation studies on $\mathrm{UCR}$ and $\mathrm{UCR\mbox{-}Aug}$ datasets. We summarize the performance of different model variations in Table \ref{table:ablation}, where
    1) the variation ``w/o graph'' denotes the representations from TCN feature encoder and length selection mechanism are used for anomaly detection without message passing;
    2) the variation ``w/o adaptive graph'' denotes the learned adaptive adjacency matrix is not used in message passing, while the $\mathbf{A}^{\mathrm{prior}}$ is used;
    3) the variation ``w/o density information'' denotes the learned adaptive adjacency matrix is not refined with density information, and the $\mathbf{A}$ is used instead of $\mathbf{\hat A}$;
    4) the variation ``w/o length selection'' denotes multi-length representations are averaged before input to DAGNN, i.e., different lengths are treated equally;
    5) the variation ``fixed length'' denotes that we do not use the multi-length representations but select a fixed subsequence length (which is set to be the period length for periodic time series and a constant $160$ for the non-periodic).

Results in Table \ref{table:ablation} demonstrate the advantages of each component in GraphSubDetector. It is important to mention that the performance of the variation ``w/o graph'' which does not perform message passing degrades a lot, suggesting the significance of robust representation learning via graph neural networks.



 \begin{figure*}[!t]
    \centering
    \includegraphics[width=0.7\textwidth]{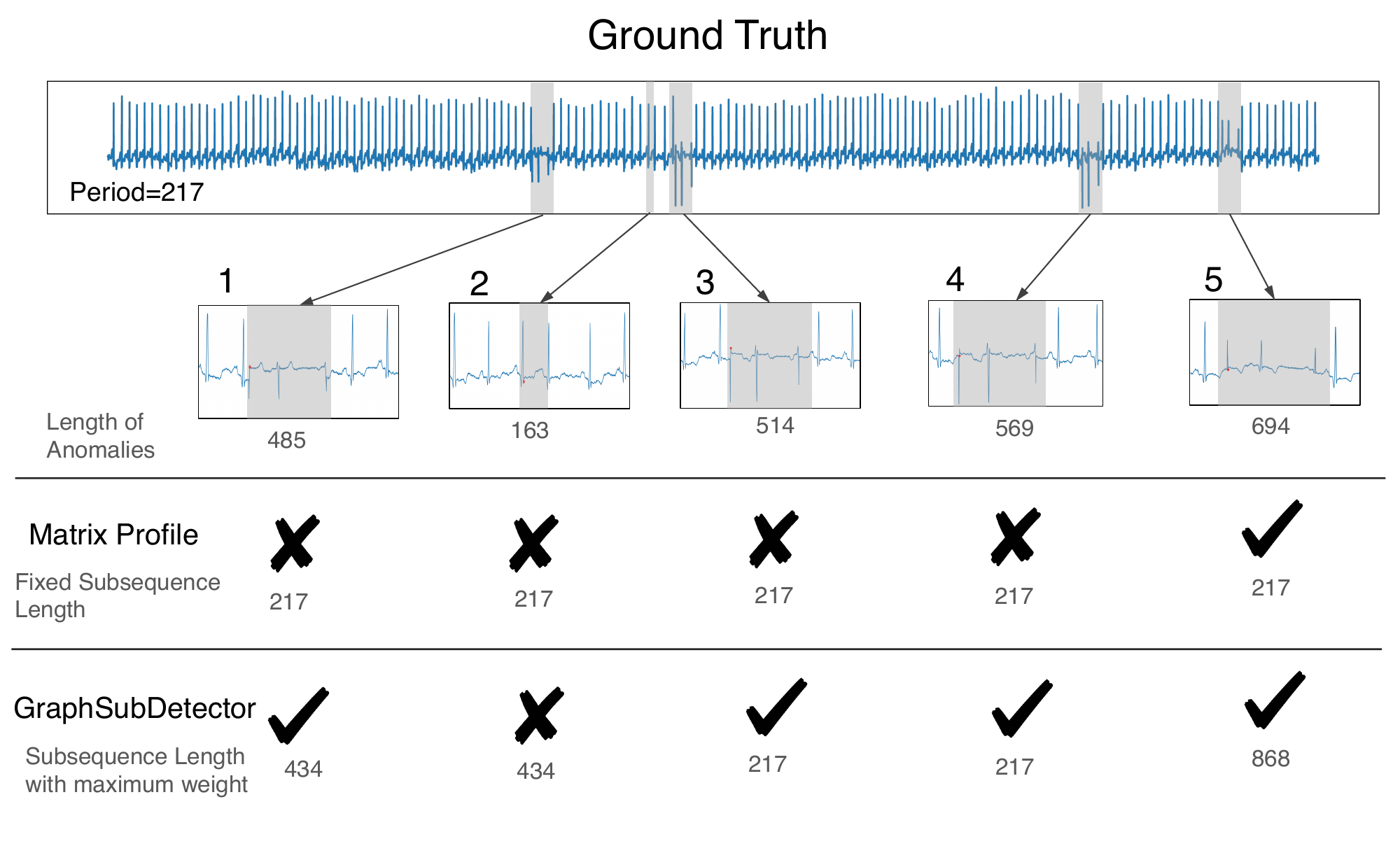}\vspace{-4mm}
    \caption{Case studies and visualization of GraphSubDetector and Matrix Profile algorithms for subsequence anomaly detection with different lengths of subsequence anomalies and recurring anomalies.}
    \label{fig:case1} 
\end{figure*}
\subsection{Case Studies and Visualization}
Here we conduct typical case studies and visualization to investigate how well GraphSubDetector addresses issues in subsequence anomaly detection problems, including handling recurring anomalies with similar patterns and selecting proper subsequence length.
Figure \ref{fig:case1} illustrates a heartbeat time series in $\mathrm{UCR\mbox{-}Aug}$. There are 5 subsequence anomalies (marked in grey), where Anomaly 3 and 4 share a similar anomalous pattern. For Matrix Profile, we set a fixed period length for its prior input length of subsequence anomaly. Matrix Profile finds 1 of 5 anomalies in its top-5 scored subsequences and it cannot deal with recurring anomalies, while GraphSubDetector finds 4 of 5 anomalies. It can be seen that the GraphSubDetector algorithm detects recurring anomalous subsequences much better than the Matrix Profile algorithm. Furthermore, in GraphSubDetector, we mark the length with the maximum weight learned by the selection mechanism. It shows that this mechanism can adaptively learn different subsequence lengths for better anomaly detection.
Therefore, we can see that the existing subsequence anomaly detection algorithm like Matrix Profile needs to set a fixed length prior for subsequence anomaly detection, which can hardly account for anomalies of different types. In contrast, the length selection mechanism of the GraphSubDetector selects appropriate subsequence lengths to achieve better detection performance.

\subsection{Model Complexity and Efficiency}
In this subsection, we discuss the complexity of GraphSubDetector. In the architecture of GraphSubDetector, kNN-graph learning is adopted. The graph structure in the learning phase is fixed as we initialize it with time series discord as prior, and only top-k important edges are preserved for each node. As a result, the complexity of learning the adjacency matrix is $O(kN)$ (where N is the number of nodes), which is linear. 
Meanwhile, the scalability of GraphSubDetector is also evaluated and plotted in Figure~\ref{fig:Model efficiency}. Since the time performance of GraphSubDetector is independent of the data quality or any other inputs except the number of subsequences, our experiment uses fixed periodicity and stride parameters, thus the number of subsequences is proportional to the data length. 
From Figure~\ref{fig:Model efficiency}, we can see that the processing time of GraphSubDetector roughly increases linearly with respect to the data length, and the processing time is less than 45 seconds for time series with length 300k. Therefore, GraphSubDetector is efficient and applicable for time series subsequence anomaly detection in most real-world scenarios.

Moreover, we compare the running time of GraphSubDectector and existing time series anomaly detection algorithms in Table~\ref{table_runningtime}. As different methods work on different devices (CPU or GPU), here we unrigorously report the running time of different methods which worked on a time series of length 34.3k. Deep models run on Nvidia Tesla-V100-32G, and non-deep models run on a 2.3GHz@12cores CPU. It can be seen that GraphSubDectector is faster than previous works. We also investigate the different key modules in GraphSubDectector for running time. It can be seen that the two modules only introduce a small amount of computational overhead, as the length selection mechanism only additionally uses multiple pooling and dot product operations and the DAGNN additionally uses vector sum operation. These two modules of GraphSubDectector are important in realistic applications, as they are designed for complicated anomalies with variable types and lengths which are rare but significant.

\begin{figure}[t]
\begin{center}
\includegraphics[width=0.43\textwidth]{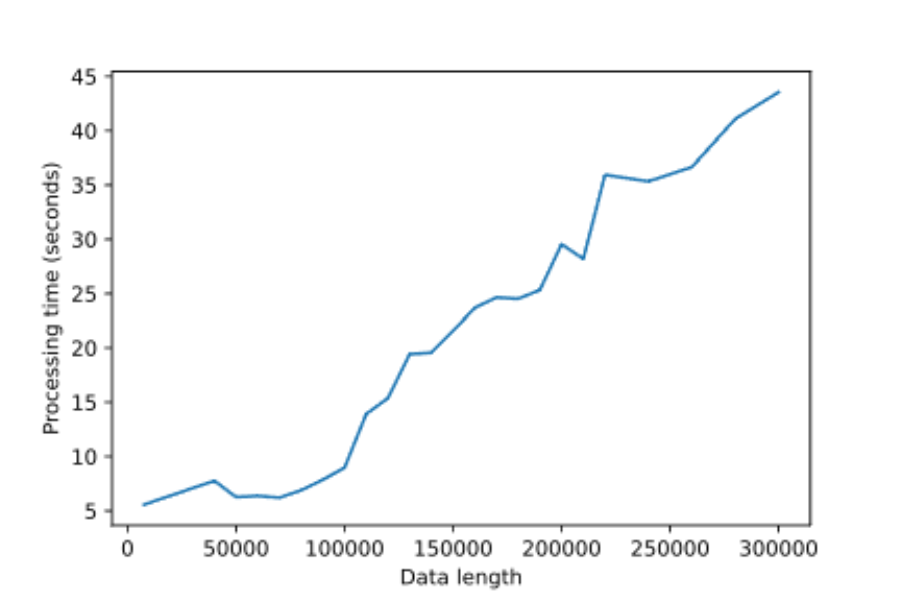}
\end{center}
\caption{Model efficiency of GraphSubDetector.}
\label{fig:Model efficiency}
\end{figure}

\begin{table}[t]
\centering
\caption{Running time comparison of different algorithms.}
\scalebox{0.8}{
 \begin{tabular}{c| c }
    \toprule
        Algorithm & Running time (seconds)   \\
    \midrule
        LOF & 479.1 \\
        Matrix Profile & 317.2  \\
        NCAD  & 68.4   \\
            \midrule
        GraphSubDetector w/o length selection mechanism & 19.8   \\
        GraphSubDetector w/o density-aware GNN & 21.2   \\
        GraphSubDetector & 23.2 \\
    \bottomrule
    \end{tabular}\label{table_runningtime}
}
\end{table}

\subsection{Density-Aware GNN}

Figure \ref{fig_dagnn} plots the subsequence representations of GraphSubDetector by t-SNE~\cite{van2008visualizing} to demonstrate the benefits of the density-aware GNN (DAGNN) component in GraphSubDetector.
Results in Figure \ref{fig_dagnn}(a) are without GNN, results in Figure \ref{fig_dagnn}(b) use adjacency matrix corresponding to Theorem 4.1 and results in Figure \ref{fig_dagnn}(c) correspond to Theorem 4.2. The normal representations of adaptive GNN output have more concentration than the model without GNN, and the proposed DAGNN can generate more compact normal representations which we call noise removal. Besides, the anomalies are relatively distant to the normal samples with DAGNN. The results are consistent with Theorem 4.1 and 4.2. Thus, representations corresponding to normal and abnormal samples using DAGNN in GraphSubDetector are more discriminative, eventually facilitating anomaly detection and improving performance.

\begin{figure}[h]
	\centering
	\subfloat[Representations without GNN]{
		\begin{minipage}[b]{0.4\linewidth}
			\includegraphics[width=1\textwidth]{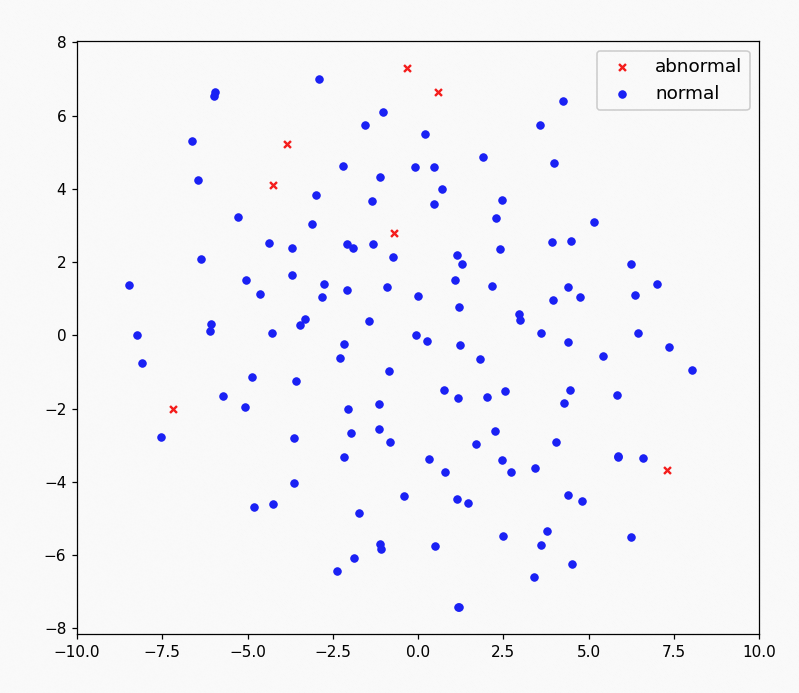}
		\end{minipage}
	}
    \quad \quad 
	\subfloat[Representations with Adaptive GNN (no density information)]{
		\begin{minipage}[b]{0.4\linewidth}
			\includegraphics[width=1\textwidth]{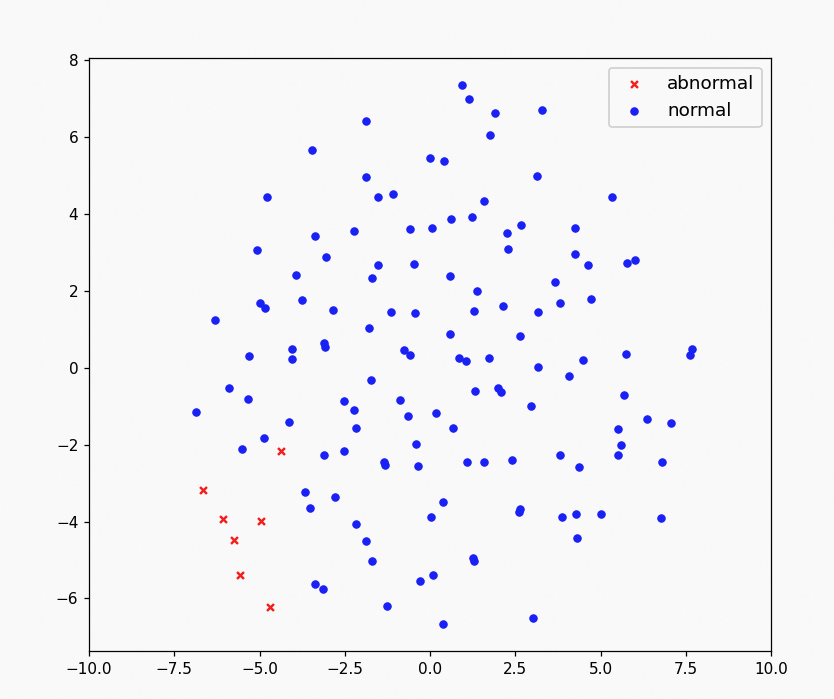}
		\end{minipage}
	}
    \quad
    \subfloat[Representations with Density-Aware GNN]{
		\begin{minipage}[b]{0.4\linewidth}
			\includegraphics[width=1\textwidth]{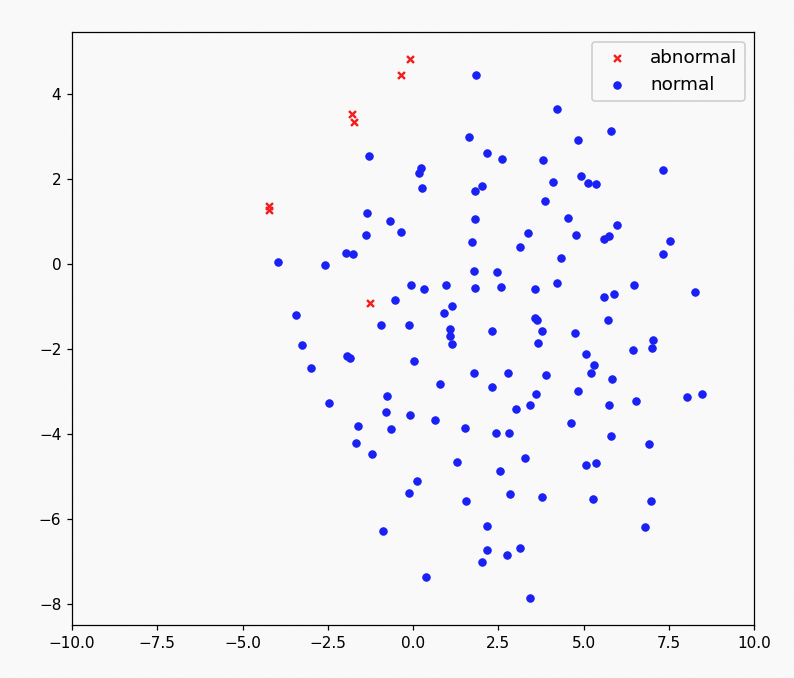}
		\end{minipage}
	}
	\caption{A t-SNE plot of the subsequence representations of GraphSubDetector. Blue dots denote normal samples, and red crosses denote abnormal samples.}
	\label{fig_dagnn}
\end{figure}

\subsection{Performance Variance}

In this part, we present standard deviations of GraphSubDetector's results and report the comparison with the best baseline (NormA) as shown in Table \ref{table:public_dats} of Section~\ref{sec_expe}. We repeat the experiments 5 times with different random seeds. It can be seen that GraphSubDetector has low variance with stable performance.

\begin{table*}[t]
\centering
\caption{Performance variance of GraphSubDetector. The best ones are in bold.}
\scalebox{0.75}{
 \begin{tabular}{c| cc|cc|cc |cc|cc }
    \toprule
       Dataset & \multicolumn{2}{|c|}{$\mathrm{UCR}$} &\multicolumn{2}{|c|}{$\mathrm{UCR\mbox{-}Aug}$} &\multicolumn{2}{|c|}{$\mathrm{SED}$} &\multicolumn{2}{|c}{$\mathrm{IOPS}$}&\multicolumn{2}{|c}{$\mathrm{ECG}$}\\
    \midrule
      Metric  &AUC  & VUS     & AUC & VUS      &AUC & VUS     & AUC & VUS        & AUC & VUS   \\
    \midrule
      NORMA & {0.874} & {0.872} & {0.942} & {0.935} & \textbf{0.998} & \textbf{0.994} & {0.982} & \textbf{0.978} & 0.576& 0.569  \\
      \!GraphSubDetector\! & \textbf{0.913  \!$\pm$\! 0.02} & \textbf{0.907 \!$\pm$\! 0.02} & \textbf{0.956 \!$\pm$\! 0.02} & \textbf{0.945 \!$\pm$\! 0.03}& 0.995  \!$\pm$\! 0.01 & 0.991  \!$\pm$\!0.00 & \textbf{0.985  \!$\pm$\! 0.02} & {0.967   \!$\pm$\! 0.02} & \textbf{0.923 \!$\pm$\! 0.03} & \textbf{0.919  \!$\pm$\! 0.05}\\
    \bottomrule
    \end{tabular}\label{table:public_dats}
}
\end{table*}

\begin{table*}[t]
    \centering
     \caption{Sensitivity of GraphSubDetector for regularization $\mathcal{L}_{\mathrm{dec}}$ and $\mathcal{L}_{\mathrm{Len}}$ in model training.}
    \scalebox{0.85}{
    \begin{tabular}{c|c |c |c |c|c|c|c|c|c|c|c|c|c}
        \toprule
        $\lambda$ & 0 & 0 &0 &0 &0 &0  &0.1 &0.2 &0.5 &1 &2 &1 &1 \\
        $\mu$ & 0 & 0.1 &0.2 &0.5 &1 &2 & 0 &0 &0 &0 &0 &1 &0.2  \\
        \midrule
         $\mathrm{UCR}$ & 0.894 &0.902 &0.913 &0.911 &0.913 &0.909 &0.896 &0.887 &0.907 &0.9 &0.892 &0.909 &0.913 \\
          $\mathrm{SMD}$ & 0.901 &0.894 &0.885 &0.911 &0.911 &0.899 &0.894 &0.891 &0.902 &0.891 &0.905 &0.914 &0.912 \\
        \bottomrule
    \end{tabular}
}
    \label{tab:model training sensitivity}
\end{table*}

\subsection{Hyperparameter Sensitivity}
Our proposed GraphSubDetector involves a few parameters and model specifications. We use $\mathrm{ECG}$ dataset to examine the hyperparameter sensitivity, including the impact of hidden dimensions $d$, the number of neighbors $K$, the initial maximum subsequence length $T$, and the number of GNN layers. The results are plotted in Figure~\ref{fig:exp sensitivity}. 
As for the hyperparameters $\lambda$ and $\mu$ in Eq. \ref{eq_training}, we run sensitivity experiments on $\mathrm{UCR}$ and $\mathrm{SMD}$ datasets for examples and report the numerical results in Table~\ref{tab:model training sensitivity}.

It can be seen that the GraphSubDetector can maintain its performance with a wide range of parameter values. Also, Figure~\ref{fig:exp sensitivity} and Table~\ref{tab:model training sensitivity} provide us insight into selecting the hyperparameter configuration for the performance-complexity tradeoff.


\begin{figure}[t]
\begin{center}
\includegraphics[width=0.5\textwidth]{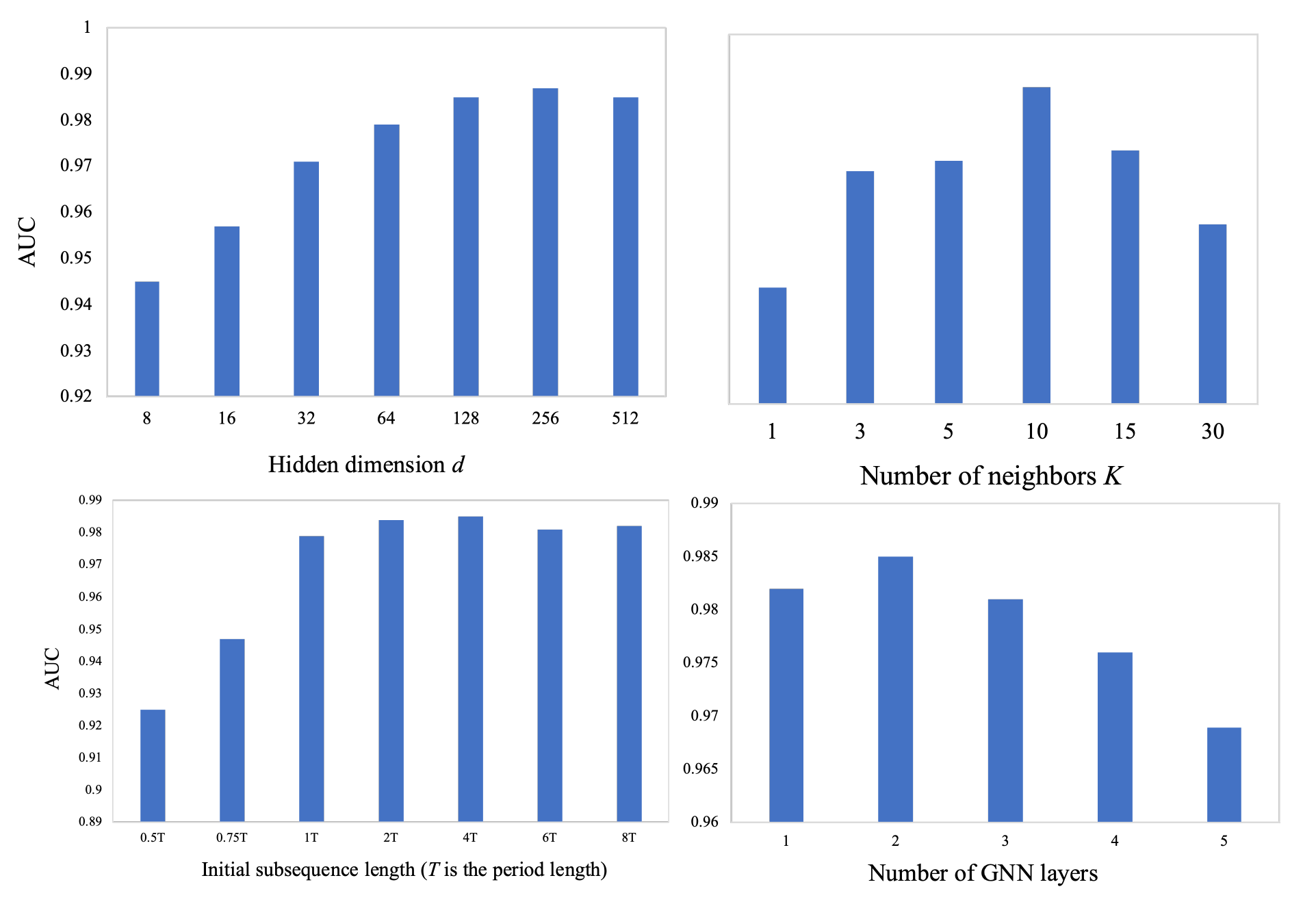}
\end{center}
\caption{Hyperparameters Sensitivity of GraphSubDetector.}
\label{fig:exp sensitivity}
\end{figure}

\subsection{Anomaly Types Injection}
In the proposed GraphSubDetector, we design an anomaly injection module to augment different types of labeled anomalous subsequences in the model learning procedure for performance improvement. In this subsection, we conduct experiments to investigate the effects of different types of injected anomalies on $\mathrm{UCR}$ and $\mathrm{UCR\mbox{-}Aug}$ datasets, and summarize the results in Table~\ref{tab:anomaly_types}. 

From Table~\ref{tab:anomaly_types}, we can find that each type of anomaly can improve performance to some extent. Overall, the GraphSubDetector achieves better performance improvements by utilizing all six different types of anomaly injection than using a single type. As discussed before, since the labeled anomaly data are rare for training in most situations, anomaly injection prompts the model to distinguish these obvious anomalies from normal data and thus generate compact normal representations, which leads to performance enhancement.

\begin{table}[h]
    \centering
    \caption{Effects of different types anomaly injection in GraphSubDetector. } 
    \scalebox{0.85}{
        \begin{tabular}{c|c |c}
        \toprule
            Datasets & $\mathrm{UCR}$ & $\mathrm{UCR\mbox{-}Aug}$  \\
            \midrule
            without anomaly injection & 0.902 & 0.924 \\
            spike$\&$dip & 0.910 & 0.943 \\
            resizing &0.907 & 0.937 \\
            warping & 0.909 &0.930 \\
            noise injection &0.907 & {0.958} \\
            left-side right &{0.911} &0.929 \\
            up-side down & 0.909 &0.938 \\
            all-types anomaly injection (GraphSubDetector) & {0.913} & {0.956} \\
        \bottomrule
        \end{tabular}
    }
    \label{tab:anomaly_types}
\end{table}



\section{Conclusion}

In this paper, we aim to bridge the gap between time series discords and deep learning-based TSAD methods and propose a novel subsequence anomaly detection approach named GraphSubDetector. 
GraphSubDetector innovates subsequence anomaly detection from two perspectives: 1) GraphSubDetector adaptively learns a proper subsequence length through a length selection mechanism by considering multi-length views of subsequences that highlight the data characteristics; 2) With density-aware adaptive graph neural networks (DAGNN), GraphSubDetector can further learn robust subsequence representations against variance of normal data. Both improvements benefit the subsequence anomaly detection problem.
Experiments on multiple benchmark datasets demonstrate the superior performance of the proposed model. 


\newpage
\bibliographystyle{ACM-Reference-Format}
\bibliography{reference}


\newpage
\appendix
\onecolumn


\section{Appendix}
\subsection{Theoretical Analysis} \label{appendix:theory}

Let $N, M$ denote the number of normal and anomalous data samples, respectively, $N \gg M$, and $d$ denotes the dimension of the features. Assume that normal samples $\mathbf{f}_1^{\mathrm{nor}}, \cdots, \mathbf{f}_N^{\mathrm{nor}} \sim \mathcal{N}(\boldsymbol{\mu}, \sigma^2)$ and anomalous samples deviate from normal ones. The generation procedure of anomalous samples is $\mathbf{f}_i^{\mathrm{ano}} = \mathbf{f}_i^{\mathrm{ref}} + \boldsymbol{\epsilon}_i$, where $1 \leq i \leq M$, $\mathbf{f}_i^{\mathrm{ref}} \sim \mathcal{N}(\boldsymbol{\mu}, \sigma^2)$ is the reference normal sample of the $i$-th anomaly, $\| \boldsymbol{\epsilon}_i \| = K \sigma$ is the corresponding nontrivial anomalous deviation, and $K > 0$. 

\begin{theorem} \label{theorem_1_copy}
 Define a fully-connected graph $\mathcal{G}$ of data samples with adjacency matrix $\mathbf{A} \in \mathbb{R} ^ {(N+M) \times (N+M)}$ encoding similarity measurement of pair-wise samples, and the entry $\mathbf{A}_{ij} = \exp{(-\|\mathbf{f}_i - \mathbf{f}_j \| ^ 2 / \delta})$, which transforms Euclidean distance between features with a Gaussian kernel. Denote $\mathbf{F} \in \mathbb{R} ^{(N+M) \times d}$ as the feature matrix of all samples, $\mathbf{F}_{\star \mathcal{G}} = \mathbf{D}^{-1}\mathbf{A}\mathbf{F}$ is the feature matrix after message passing w.r.t. graph $\mathcal{G}$, and $\sigma_{\star \mathcal{G}}$ is the standard deviation of the newly generated features. 
 Then for any $K > 0$, there exist  $\delta > 0$ that for each anomalous sample
 \begin{equation}
     \frac{\| \mathbf{f}_{i, \star \mathcal{G}}^{\mathrm{ano}} - \mathbf{f}_{i, \star \mathcal{G}}^{\mathrm{ref}} \|}{\sigma_{\star \mathcal{G}}} > \frac{\|\mathbf{f}_i^{\mathrm{ano}} - \mathbf{f}_i^{\mathrm{ref}} \|}{\sigma}.
 \end{equation}
\end{theorem}

\begin{theorem} \label{theorem_2_copy}
 Define fully connected graph $\mathcal{\hat{G}}$ with adjacency matrix $\mathbf{\hat{A}}'$, and $\mathbf{\hat{A}}_{ij} = \exp{(-\|\mathbf{f}_i - \mathbf{f}_j \| ^ 2 / \delta} - \| \mathbf{f}_i - \boldsymbol{\mu} \| ^ 2 / c ) $. $\mathbf{F}_{\star \mathcal{\hat{G}}} = \mathbf{\hat{D}}^{-1}\mathbf{\hat{A}}\mathbf{F}$ is the feature matrix after message passing w.r.t. graph $\mathcal{\hat{G}}$, and $\sigma_{\star \mathcal{\hat{G}}}$ is the standard deviation of the newly generated features. 
 Then for any $\delta > 0$, there exist  $c > 0$ that for each anomalous sample
 \begin{equation}
     \frac{\| \mathbf{f}_{i, \star \mathcal{\hat{G}}}^{\mathrm{ano}} - \mathbf{f}_{i, \star \mathcal{\hat{G}}}^{\mathrm{ref}} \|}{\sigma_{\star \mathcal{\hat{G}}}} > \frac{\| \mathbf{f}_{i, \star \mathcal{G}}^{\mathrm{ano}} - \mathbf{f}_{i, \star \mathcal{G}}^{\mathrm{ref}} \|}{\sigma_{\star \mathcal{G}}}.
 \end{equation}
\end{theorem}
\begin{proof} of Theorem \ref{theorem_1_copy} and \ref{theorem_2_copy} 

First, as $N \gg M$, the anomalous samples will not influence the statistics of the joint set of data samples. Next we analyze a specific anomaly sample $\mathbf{f}^{\mathrm{ano}} = \mathbf{f}^{\mathrm{ref}} + \boldsymbol{\epsilon}$. Let $a^{\mathrm{ano}}_j = \exp{(-\|\mathbf{f}^{\mathrm{ano}} - \mathbf{f}_j \| ^ 2 / \delta})$ and $a^{\mathrm{ref}}_j = \exp{(-\|\mathbf{f}^{\mathrm{ref}} - \mathbf{f}_j \| ^ 2 / \delta})$ satisfying the above similarity definition. As the anomalous data error $\boldsymbol{\epsilon}$ is independent with $\mathbf{f}^{\mathrm{ref}} - \mathbf{f}_j$, in high-dimensional feature space, the inner product $\boldsymbol{\epsilon}(\mathbf{f}^{\mathrm{ref}} - \mathbf{f}_j) \approx 0$ holds.
Thus, we have
 \begin{align}
    a^{\mathrm{ano}}_j &= \exp{(-\|\mathbf{f}^{\mathrm{ano}} - \mathbf{f}_j \| ^ 2 / \delta)} \notag\\
    &= \exp{(-\|\mathbf{f}^{\mathrm{ref}} - \mathbf{f}_j - \boldsymbol{\epsilon} \| ^ 2 / \delta)} \notag\\
    &\approx a^{\mathrm{ref}}_j  \exp{(- (K \sigma)^2 / \delta)}.
 \end{align}
Assume that the independence property of anomalous data error still holds after message passing. We can further derive that,
 \begin{align} \notag
    \| \mathbf{f}_{\star \mathcal{G}}^{\mathrm{ano}} - \mathbf{f}_{\star \mathcal{G}}^{\mathrm{ref}} \| ^ 2 \!&=\!
    \left\| \frac{ \mathbf{f}^{\mathrm{ano}}+\sum_{j=1}^{N} a^{\mathrm{ano}}_j \mathbf{f}_j}{1+\sum_{j=1}^{N} a^{\mathrm{ano}}_j} \!-\!  \frac{ \mathbf{f}^{\mathrm{ref}}+\sum_{j=1}^{N} a^{\mathrm{ref}}_j \mathbf{f}_j}{1+\sum_{j=1}^{N} a^{\mathrm{ref}}_j}\right\|^2  \notag\\
   \!& >\! \frac{\left\|\mathbf{f}^{\mathrm{ano}} - \mathbf{f}^{\mathrm{ref}} \right\|^2 + \left \| \sum_j (1 - \mathrm{e} ^ {-\frac{K^2\sigma^2}{\delta}}) a_j^{\mathrm{ref}} \mathbf{f}_j \right \|^2}{(1+\sum_j a_j^{\mathrm{ref}})^2}  \notag\\
       \!&   > \! \frac{\left\|\mathbf{f}^{\mathrm{ano}} - \mathbf{f}^{\mathrm{ref}} \right\|^2 \!+\! \left \| \sum_j (1 - \mathrm{e} ^ {-\frac{K^2\sigma^2}{\delta}}) a_{\mathrm{min}}^{\mathrm{ref}} \mathbf{f}_j \right \|^2}{(1+N)^2},\notag
\end{align}
where $a_{\mathrm{min}}^{\mathrm{ref}}$ is the $\mathbf{f}^{\mathrm{ref}}$'s minimum similarity to normal data samples, which is not infinitely small value with a high probability. Thus, when 
 \begin{equation}
    0< \delta < \frac{-K^2}{\log \left(1-\sqrt{\frac{(N^2+2N)(\|\boldsymbol{\mu}\|^2+\sigma^2)}{N^2(\|\boldsymbol{\mu}\|^2+K^2\sigma^2)}}\right)}
 \end{equation}
 holds, we have $\| \mathbf{f}_{\star \mathcal{\hat{G}}}^{\mathrm{ano}} - \mathbf{f}_{ \star \mathcal{\hat{G}}}^{\mathrm{ref}} \| < \| \mathbf{f}^{\mathrm{ano}} - \mathbf{f}^{\mathrm{ref}} \|$ hold.
 On the other hand, as message passing performs as a feature average, the feature variance will be reduced. Thus, Theorem \ref{theorem_1_copy} holds. The variance of $\sigma_{\star \mathcal{\hat G}}$ in Theorem \ref{theorem_2_copy} are further reduced, by moving all samples to $\boldsymbol{\mu}$, and thus Theorem \ref{theorem_2_copy} holds.
\end{proof}

\end{document}